%% file: iclr2022_conference.tex
\documentclass{article} % For LaTeX2e
\usepackage{iclr2022_conference,times}

\input{math_commands.tex}

\usepackage{url}

\usepackage{amssymb}
\usepackage{amsthm}
\usepackage{mathtools}
\usepackage{empheq}
\usepackage{subcaption}
\usepackage{wrapfig}
\usepackage{graphicx}
\usepackage{xspace}
\usepackage{mathtools}
\usepackage{algorithm}
\usepackage[noend]{algorithmic}
\usepackage{booktabs}
\usepackage{bbm}
\usepackage{hyperref}
\usepackage{url}
\usepackage[compact]{titlesec} 
\hypersetup{colorlinks=true,linkcolor=black,citecolor=brown,urlcolor=blue}

\include{defs}

\def\method{TRAIL\xspace}

\title{TRAIL: Near-Optimal Imitation Learning with Suboptimal Data}  
\author{%
  Mengjiao Yang \\ UC Berkeley, Google Brain \\ \texttt{sherryy@google.com}
  \And Sergey Levine \\ UC Berkeley, Google Brain
  \And Ofir Nachum \\ Google Brain
}

\iclrfinalcopy % Uncomment for camera-ready version, but NOT for submission.
\begin{document}

\maketitle

\vspace{-3em}
\begin{abstract}
The aim in imitation learning is to learn effective policies by utilizing near-optimal expert demonstrations. However, high-quality demonstrations from human experts can be expensive to obtain in large number. On the other hand, it is often much easier to obtain large quantities of suboptimal or task-agnostic trajectories, which are not useful for direct imitation, but can nevertheless provide insight into the dynamical structure of the environment, showing what \emph{could} be done in the environment even if not what \emph{should} be done. We ask the question, is it possible to utilize such suboptimal offline datasets to facilitate \emph{provably} improved downstream imitation learning? In this work, we answer this question affirmatively and present training objectives that use offline datasets to learn a \emph{factored} transition model whose structure enables the extraction of a \emph{latent action space}. Our theoretical analysis shows that the learned latent action space can boost the sample-efficiency of downstream imitation learning, effectively reducing the need for large near-optimal expert datasets through the use of auxiliary non-expert data. To learn the latent action space in practice, we propose \method (Transition-Reparametrized Actions for Imitation Learning), an algorithm that learns an energy-based transition model contrastively, and uses the transition model to reparametrize the action space for sample-efficient imitation learning. We evaluate the practicality of our objective through experiments on a set of navigation and locomotion tasks. Our results verify the benefits suggested by our theory and show that \method is able to improve baseline imitation learning by up to 4x in performance.\footnote{Find experimental code at~\url{https://github.com/google-research/google-research/tree/master/rl_repr}.}
\end{abstract}

\setlength{\abovedisplayskip}{2pt}
\setlength{\abovedisplayshortskip}{2pt}
\setlength{\belowdisplayskip}{2pt}
\setlength{\belowdisplayshortskip}{2pt}
\setlength{\jot}{2pt}
\setlength{\floatsep}{2ex}
\setlength{\textfloatsep}{2ex}
\setlength{\parskip}{0.1em}
\titlespacing\section{0pt}{10pt plus 2pt minus 2pt}{0pt plus 2pt minus 2pt}
\titlespacing\subsection{2pt}{10pt plus 2pt minus 2pt}{2pt plus 2pt minus 2pt}

\input{introduction}
\input{related}
\input{background}

\input{method}
\input{experiment}
\input{conclusion}

\subsubsection*{Acknowledgments}
We thank Dale Schuurmans and Bo Dai for valuable discussions. We thank Justin Fu, Anurag Ajay, and Konrad Zolna for assistance in setting up evaluation tasks.

\bibliography{iclr2022_conference}
\bibliographystyle{iclr2022_conference}

\appendix
\input{appendix}
\end{document}

%% file: math_commands.tex
\usepackage{amsmath,amsfonts,bm}

\def\Figref#1{Figure~\ref{#1}}

\def\Secref#1{Section~\ref{#1}}
\def\eqref#1{equation~\ref{#1}}
\def\Eqref#1{Equation~\ref{#1}}
\def\1{\bm{1}}

\DeclareMathAlphabet{\mathsfit}{\encodingdefault}{\sfdefault}{m}{sl}
\SetMathAlphabet{\mathsfit}{bold}{\encodingdefault}{\sfdefault}{bx}{n}

\def\gX{{\mathcal{X}}}

\newcommand{\E}{\mathbb{E}}

\newcommand{\R}{\mathbb{R}}

\DeclareMathOperator*{\argmin}{arg\,min}

%% file: defs.tex
\def\sqrtexplained#1{%
  \begingroup
    \sbox0{$#1$}
    \def\underbrace##1_##2{##1}
    \sbox2{$#1$}
    \dimen0=\wd0 \advance\dimen0-\wd2
    \mathrlap{\sqrt{\phantom{\displaystyle#1}\kern\dimen0 }}
    \hphantom{\sqrt{\vphantom{\displaystyle#1}}}
  \endgroup
  #1}

\def\mdp{\mathcal{M}}
\def\pitarget{\pi_*}
\def\visitpi{d^\pi}

\def\visittarget{d^{\pitarget}}

\def\visitrb{d^\mathrm{off}}

\def\Sset{S}
\def\Aset{A}
\def\Zset{Z}
\def\Dset{\mathcal{D}}
\def\Demos{\mathcal{D}^{\pitarget}}
\def\Doff{\mathcal{D}^\mathrm{off}}
\def\Reward{\mathcal{R}}
\def\Trans{\mathcal{T}}
\def\init{\mu}
\def\defeq{:=}
\renewcommand{\widehat}{\hat}

\def\opt{\mathcal{OPT}}
\def\sopt{opt}

\def\Tmodel{\overline{\mathcal{T}}}

\def\Trep{\mathcal{T}_\Zset}
\def\pirep{\pi_{\Zset}}
\def\pilrep{\pi_\theta}
\def\pidec{\pi_{\alpha}}

\def\jbc{J_\mathrm{BC}}
\def\jbcrep{J_{\mathrm{BC},\phi}}
\def\jbcdec{J_{\mathrm{DE}}}

\def\jtrans{J_\mathrm{T}}

\def\Diff{\mathrm{Diff}}
\def\err{\mathrm{Err}}
\def\dtv{D_\mathrm{TV}}
\def\dkl{D_\mathrm{KL}}
\def\dchi{D_\mathrm{\chi^2}}

\def\const{\mathrm{const}}

\def\uniform{\mathrm{Unif}_{\Aset}}

\newtheorem{theorem}{Theorem}
\newtheorem{lemma}[theorem]{Lemma}

%% file: introduction.tex
\section{Introduction}\label{sec:intro}

Imitation learning uses expert demonstration data to learn sequential decision making policies~\citep{schaal1999imitation}. Such demonstrations, often produced by human experts, can be costly to obtain in large number. On the other hand, practical application domains, such as recommendation~\citep{afsar2021reinforcement} and dialogue~\citep{jiang2021towards} systems, provide large quantities of offline data generated by suboptimal agents. Since the offline data is suboptimal in performance, using it directly for imitation learning is infeasible. While some prior works have proposed using suboptimal offline data for offline reinforcement learning (RL) ~\citep{kumar2019stabilizing,wu2019behavior,levine2020offline}, this would require reward information, which may be unavailable or infeasible to compute from suboptimal data~\citep{abbeel2004apprenticeship}.
Nevertheless, conceptually, suboptimal offline datasets should contain useful information about the environment, if only we could distill that information into a useful form that can aid downstream imitation learning.

One approach to leveraging suboptimal offline
%%SL.10.3: suboptimal/background/etc.? I think we can come up with a term that is more informative than "offline data", my suggestion would either be "suboptimal offline data" or "background data". Whatever term we choose to use, we should clearly explain its distinction from the expert demonstration data.
datasets is to use the offline data to extract a lower-dimensional \emph{latent action space}, and then perform imitation learning on an expert dataset using this latent action space. If the latent action space is learned properly, one may hope that performing imitation learning in the latent space can reduce the need for large quantities of expert data.
%%SL.10.3: OK, so the purpose of this paragraph I think is to explain why the problem is *hard*. In that case, we should more explicitly position this, e.g. something like this: 
While a number of prior works have studied similar approaches in the context of hierarchical imitation and RL setting~\citep{parr1998reinforcement,dietterich1998maxq,sutton1999between,kulkarni2016hierarchical,vezhnevets2017feudal,nachum2018data,ajay2020opal,pertsch2020accelerating,hakhamaneshi2021hierarchical}, such methods typically focus on the theoretical and practical benefits of \emph{temporal abstraction} by extracting temporally extended skills from data or experience. That is, the main benefit of these approaches is that the latent action space operates at a lower temporal frequency than the original environment action space. We instead focus directly on the question of \emph{action representation}: instead of learning skills that provide for temporal abstraction, we aim to directly reparameterize the action space in a way that provides for more sample-efficient downstream imitation without the need to reduce control frequency. Unlike learning temporal abstractions, action reparamtrization does not have to rely on any hierarchical structures in the offline data, and can therefore utilize highly suboptimal datasets (e.g., with random actions).
Aiming for a provably-efficient approach to utilizing highly suboptimal offline datasets, we use first principles to derive an upper bound on the quality of an imitation learned policy involving three terms corresponding to (\ref{eq:tabular-rep}) action representation and (\ref{eq:tabular-dec}) action decoder learning on a suboptimal offline dataset, and finally, (\ref{eq:tabular-bc}) behavioral cloning (i.e., max-likelihood learning of latent actions) on an expert demonstration dataset. 
%%SL.10.3: I think the issue here is that you are mainly describing the analysis tool, before you've actually described what the method is. Perhaps it would be better to first spend some time motivating the dynamics-based method for extracting action representations and how it works, and then provide a brief high-level description of what you prove about it? You don't have to actually explain *how* you prove it, just what you prove
The first term in our bound immediately suggests a practical offline training objective based on a transition dynamics loss using an \emph{factored} transition model. We show that under specific factorizations (e.g., low-dimensional or linear), one can guarantee improved sample efficiency on the expert dataset. Crucially, our mathematical results avoid the potential shortcomings of temporal skill extraction, as our bound is guaranteed to hold even when there is no temporal abstraction in the latent action space.

\begin{figure}[t]
\centering
 \includegraphics[width=0.9\linewidth]{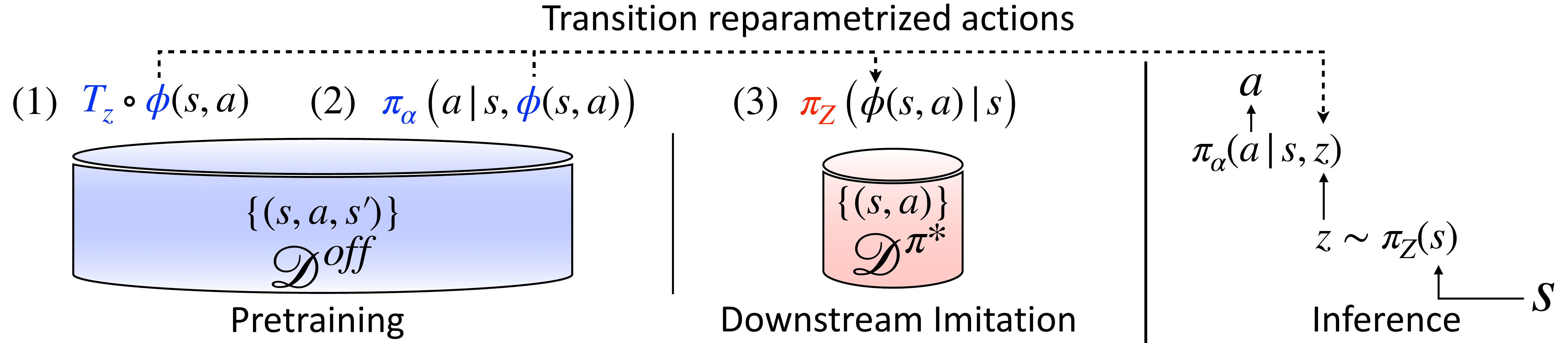}
 \vspace{-2mm}
 \caption{The TRAIL framework. Pretraining learns a factored transition model $\Trep\circ\phi$ and an action decoder $\pidec$ on $\Doff$. Downstream imitation learns a latent policy $\pirep$ on $\Demos$ with expert actions reparametrized by $\phi$. During inference, $\pirep$ and $\pidec$ are combined to sample an action.}
 \label{fig:framework}
 \vspace{-3mm}
\end{figure}

We translate these mathematical results into an algorithm that we call \emph{Transition-Reparametrized Actions for Imitation Learning} (\method). As shown in~\Figref{fig:framework}, \method consists of a pretraining stage (corresponding to the first two terms in our bound) and a downstream imitation learning stage (corresponding to the last term in our bound). During the pretraining stage, \method uses an offline dataset to learn a factored transition model and a paired action decoder. During the downstream imitation learning stage, \method first reparametrizes expert actions into the latent action space according to the learned transition model, and then learns a latent policy via behavioral cloning in the latent action space. During inference, \method uses the imitation learned latent policy and action decoder in conjunction to act in the environment. In practice, \method parametrizes the transition model as an energy-based model (EBM) for flexibility and trains the EBM with a contrastive loss. The EBM enables the low-dimensional factored transition model referenced by our theory, and we also show that one can recover the \emph{linear} transition model in our theory by approximating the EBM with random Fourier features~\citep{rahimi2007random}.
%By approximating the EBM with a linear transition model through random features~\citep{rahimi2007random}, \method is guaranteed to achieve near-optimal imitation performance in domains with continuous action space and stochastic expert policies.

%One concern around latent action extraction is whether the latent space is expressive enough to capture any stochastic policies. By approximating the dynamics EBM using a linear model, we show that any stochastic policies can be represented as a deterministic policy in the latent action space, which further simplies downstream learning.
%%SL.9.24: This comes across as not very well motivated -- before describing the EBM stuff, can you somehow motivate why this is a good choice?

%%SL.9.24: Start with "Our main contribution is..." -- make it very clear to readers what the new exciting and novel thing is
%%SL.10.3: Still need to state clearly precisely what the contributions are, at least somewhere in the intro, otherwise reviewers may get confused and not understand what is novel. The "nothing here is novel" criticism is extremely difficult to deflect, and once someone says it, it's very hard to deal with in a rebuttal.
To summarize, our contributions include (i) a provably beneficial objective for learning action representations without temporal abstraction and (ii) a practical algorithm for optimizing the proposed objective by learning an EBM or linear transition model.
An extensive evaluation on a set of navigation and locomotion tasks demonstrates the effectiveness of the proposed objective.
%of learning dynamics-based action representations compared to VAE-based
%%SL.9.24: I think this comes across as a little too "apologetic" -- you can talk about VAEs later, you don't have to belabor this point in the introduction like that
\method's empirical success compared to a variety of existing methods suggests that the benefit of learning \emph{single-step} action representations has been overlooked by previous temporal skill extraction methods. Additionally, \method significantly improves  behavioral cloning even when the offline dataset is unimodal or highly suboptimal (e.g., obtained from a random policy), whereas temporal skill extraction methods lead to \emph{degraded} performance in these scenarios. Lastly, we show that \method, without using reward labels, can perform similarly or better than offline reinforcement learning (RL) with orders of magnitude less expert data, suggesting new ways for offline learning of squential decision making policies.

%% file: related.tex
\section{Related Work}
Learning action abstractions is a long standing topic in the hierarchical RL literature~\citep{parr1998reinforcement,dietterich1998maxq,sutton1999between,kulkarni2016hierarchical,nachum2018data}. A large body of work focusing on \emph{online skill discovery} have been proposed as a means to improve exploration and sample complexity in online RL. For instance, \citet{eysenbach2018diversity,sharma2019dynamics,gregor2016variational,warde2018unsupervised,liu2021learn} propose to learn a diverse set of skills by maximizing an information theoretic objective. Online skill discovery is also commonly seen in a hierarchical framework that learns a continuous space~\citep{vezhnevets2017feudal,hausman2018learning,nachum2018data,nachum2019multi} or a discrete set of lower-level policies~\citep{bacon2017option,stolle2002learning,peng2019mcp}, upon which higher-level policies are trained to solve specific tasks. Different from these works, we focus on learning action representations \emph{offline} from a fixed suboptimal dataset to accelerate imitation learning.
%Because of the online nature of these work, they are less likely to encounter the aforementioned degenerate skills problem.
%%SL.10.3: I feel like we shouldn't lean so much on the "degerate skills problem" -- it comes across as a rather esoteric technical detail, that many readers won't actually understand.
%Nor are they limited by the quality of skills in the replay buffer, as new data with diverse behaviors could be generated from interacting with the environment.
%%SL.10.3: While the above are all true statements about *prior* work, it's very important for each paragraph in the related work section to precisely explain how *your* work is different. So instead of saying "they are not this" better say "in contrast, our method [does this/handles this/etc]". Basically, the reader doesn't know what your method does yet, so if you just say what prior work doesn't do, they won't get that the reason you are saying it is because your work *does* do it

Aside from online skill discovery, \emph{offline skill extraction} focuses on learning temporally extended action abstractions from a fixed offline dataset. Methods for offline skill extraction generally involve maximum likelihood training of some latent variable models on the offline data, followed by downstream planning~\citep{lynch2020learning}, imitation learning~\citep{kipf2019compile,ajay2020opal,hakhamaneshi2021hierarchical}, offline RL~\citep{ajay2020opal}, or online RL~\citep{fox2017multi,krishnan2017ddco,shankar2020learning,shankar2019discovering,singh2020parrot,pertsch2020accelerating,pertsch2021guided,wang2021skill} in the induced latent action space. Among these works, those that provide a theoretical analysis attribute the benefit of skill extraction predominantly to increased temporal abstraction as opposed to the learned action space being any ``easier'' to learn from than the raw action space~\citep{ajay2020opal,nachum2018near}. Unlike these methods, our analysis focuses on the advantage of a lower-dimensional reparametrized action space agnostic to temporal abstraction. Our method also applies to offline data that is highly suboptimal (e.g., contains random actions) and potentially unimodal (e.g., without diverse skills to be extracted),
%In reality, however, highly suboptimal offline datasets may not contain any temporally extended skills (e.g., dataset of random actions). An additional shortcoming of these existing methods is that na\"{i}ve maximum likelihood training of observed actions can result in degenerate latent modes~\citep{zheng2018degeneration} in which latent actions are ignored by a flexible action decoder, and this is prone to occur when the offline data is unimodal, e.g., collected by a single stationary policy~\citep{ajay2020opal}.
%These work assume that the offline dataset contains skills useful for downstream tasks or contains diverse skills, and are unlikely
%%SL.10.3: Well, this is a statement about empirical performance, and that requires proof, which the current paper doesn't have. You could rephrase this to instead talk about what your method does better, but be careful and avoid making factual statements that are not backed up by proof either in your paper or in prior papers.
%to learn effectively from highly suboptimal offline data. 
which have been considered challenging by previous work~\citep{ajay2020opal}. %accelerates downstream imitation learning for reasons that are unrelated to temporal abstraction.
%%SL.9.14: why is this good?

While we focus on reducing the complexity of the action space through the lens of action representation learning,
there exists a disjoint set of work that focuses on accelerating RL with \emph{state} representation learning~\citep{singh1995reinforcement,ren2002state,castro2010using,gelada2019deepmdp,zhang2020learning,arora2020provable,nachum2021provable}, some of which have proposed to extract a latent state space from a learned dynamics model. Analogous to our own derivations, these works attribute the benefit of representation learning to a smaller latent state space reduced from a high-dimensional input  state space (e.g., images). 

%%SL.9.14: very important to provide great coverage of various hierarchy + IL methods in the literature, do a thorough Google search to track them all down
%%SL.10.3: The current related work section is better, but you might want to have at least one sentence citing classic hierarchical RL work somewhere (e.g., options). It's a bit of a lightning rod right now for angry old timer RL reviewers who will pick on you for only citing recent deep RL work, as well as for more robotics-focused imitation learning reviewers who will pick on you for ignoring the imitation learning literature.

%% file: background.tex
\section{Preliminaries}
%%SL.9.14: Need a really good problem statement section somewhere
In this section, we introduce the problem statements for imitation learning and learning-based control, and define relevant notations.
%%SL.9.24: I wonder if there is a way to organize this in a way that comes across as more of a logical progression, where we start with a brief summary of how imitation learning works (since that's the thing that follows logically from the intro), which would give us a chance to define the assumption that the expert is optimal wrt an unknown reward function, then talk about what the suboptimal data look

%%SL.9.24: Some readers will understandably a bit confused that you start with MDPs and reward functions when the goal is to do imitation learning.
\paragraph{Markov decision process.} Consider an MDP~\citep{puterman1994markov} $\mdp\defeq \langle\Sset, \Aset, \Reward, \Trans, \init, \gamma\rangle$, consisting of a state space $\Sset$, an action space $\Aset$, a reward function $\Reward:\Sset\times\Aset\to\R$, a transition function $\Trans:\Sset\times\Aset\to\Delta(\Sset)$\footnote{$\Delta(\gX)$ denotes the simplex over a set $\gX$.}, an initial state distribution $\mu\in\Delta(\Sset)$, and a discount factor $\gamma \in [0, 1)$
%%SL.9.24: This is kind of a stylistic point, but I generally prefer to stick to the notation that is standard in the community. While it is valid to represent probability distributions as mappings from variables to simplices over variables, I think in ML this is a bit obscure and needlessly convoluted. If you just say p(y|x) that gets the point across just as well and is less likely to confuse.
A policy $\pi:\Sset\to\Delta(\Aset)$ interacts with the environment starting at an initial state $s_0 \sim \init$. An action $a_t\sim\pi(s_t)$ is sampled and applied to the environment at each step $t \ge 0$. The environment produces a scalar reward $\Reward(s_t,a_t)$
%%SL.10.3: This is RL boilerplate. It would be good to edit this section to be more self-conscious about the fact that we are *not* doing RL, but rather imitation learning. Since the fact that your method doesn't require rewards is so important, it's critical to avoid any misunderstanding about this.
and transitions into the next state $s_{t+1}\sim\Trans(s_t,a_t)$. Note that we are specifically interested in the imitation learning setting, where the rewards produced by $\Reward$ are unobserved by the learner.
The state visitation distribution $\visitpi(s)$ induced by a policy $\pi$ is defined as $\visitpi(s) := (1-\gamma)\sum_{t=0}^\infty\gamma^t\cdot\Pr\left[s_t=s|\pi,\mdp\right]$. We relax the notation and use $(s,a)\sim d^\pi$ to denote $s\sim d^\pi, a\sim\pi(s)$.

%%SL.9.24: When introducing all these definitions, try to provide a little bit of signaling to the reader for *why* these definitions are being introduced. E.g., something like: In our analysis, we will [whatever], which will use [some definition]

%%SL.9.24: Perhaps it would be a good idea to more cleanly separate the problem statement from "boilerplate definitions." Part of the contribution in the paper is the problem statement, so it would be nice to separate it into its \paragraph or even \subsection and make it clear that this is the full problem statement. Then, state the problem statement completely, rather than just stating part of it like you do here. (in this case, it might also be good to retitle this section "Preliminaries and Problem Statement" or something like that)
\paragraph{Learning goal.} Imitation learning aims to recover an \emph{expert policy} $\pitarget$ with access to only a fixed set of samples from the expert:
%%SL.10.3: do you mean expert policy? "target policy" is pretty unfamiliar terminology to me
$\Demos=\{(s_i,a_i)\}_{i=1}^n$ with $s_i\sim d^\pitarget$ and $a_i\sim\pitarget(s_i)$. One approach to imitation learning is to learn a policy $\pi$ that minimizes some discrepancy between $\pi$ and $\pitarget$. In our analysis, we will use the total variation (TV) divergence in state visitation distributions,
\begin{equation}
    \Diff(\pi,\pitarget) = \dtv(d^\pi\|d^\pitarget),\nonumber
\end{equation}
as the way to measure the discrepancy between $\pi$ and $\pitarget$. Our bounds can be easily modified to apply to other divergence measures such as the Kullback–Leibler (KL) divergence or difference in expected future returns. \emph{Behavioral cloning} (BC)~\citep{pomerleau1989alvinn} solves the imitation learning problem by learning $\pi$ from $\Demos$ via a maximum likelihood objective
\begin{equation}
    \jbc(\pi) \defeq \E_{(s,a)\sim(\visittarget,\pitarget)}[-\log\pi(a|s)],\nonumber
\end{equation}
which optimizes an upper bound of $\Diff(\pi,\pitarget)$ defined above~\citep{ross2010efficient,nachum2021provable}:
\begin{equation}
\small
   \Diff(\pi, \pitarget) \le  \frac{\gamma}{1-\gamma}\sqrt{\frac{1}{2}\E_{\visittarget}[\dkl(\pitarget(s)\|\pi(s))]} =\frac{\gamma}{1-\gamma}\sqrt{\const(\pitarget) + \frac{1}{2}\jbc(\pi)}.\nonumber
   %%SL.9.24: are you sure that second line is = and not <=?
\end{equation}
\paragraph{BC with suboptimal offline data.} The standard BC objective (i.e., direct max-likelihood on $\Demos$) can struggle to attain good performance when the amount of expert demonstrations is limited~\citep{ross2011reduction,tu2021closing}. We assume access to an additional \emph{suboptimal} offline dataset
$\Doff=\{(s_i,a_i,s_i^\prime)\}_{i=1}^m$, where the suboptimality is a result of (i) suboptimal action samples $a_i\sim \uniform$
%%SL.10.3: wait, we explicitly assume that actions must be *uniform*? that seems like a pretty strong assumption -- are you sure we can't rewrite it in terms of full support or something else that is less restrictive?
and (ii) lack of reward labels. We use $(s,a,s')\sim\visitrb$ as a shorthand for simulating finite sampling from $\Doff$ via $s_i\sim\visitrb, a_i\sim \uniform,s_i'\sim\Trans(s_i,a_i),$ where $\visitrb$ is an \emph{unknown} offline state distribution. We assume $\visitrb$ sufficiently covers the expert distribution; i.e., $\visittarget(s) > 0 \Rightarrow \visitrb(s) > 0$ for all $s\in S$.
%%SL.10.3: I'm confused by this sentence, what does "simulating finite sampling" mean? Perhaps there is a way to rewrite this sentence to be more clear?
The uniform sampling of actions in $\Doff$ is largely for mathematical convenience, and in theory can be replaced with any distribution uniformly bounded from below by $\eta>0$, and our derived bounds will be scaled by $\frac{1}{|\Aset| \eta}$ as a result. This works focuses on how to utilize such a suboptimal $\Doff$ to provably accelerate BC.
%%SL.10.3: I think some of the details in this paragraph could actually be moved to the methods section. Specifically, this uniform action assumption and the agarwal citation are pretty cryptic at this point in the paper, but I suspect they'll be made much clearer once your method is more fully explained, so perhaps this would be more effective if you move this bit to later, closer to where the assumption is used, and instead just focus this paragraph primarily on stating the problem.

%$\Doff$ does not contain reward labels and could have been generated by any number of policies of any quality (i.e., $\Doff$ can exhibit unimodal, multimodal, or nearly random behavior).

%%SL.9.24: My sense is that it would be better to move this stuff into the methods section -- this is now getting into *how* we solve the problem, which doesn't belong in this section. Let's have this section just be definitions and a definition of the problem, and anything about how the problem is solved can then go into the next section

%% file: method.tex
\section{Near-Optimal Imitation Learning with Reparametrized Actions}
%%SL.10.3: Generally, I think the organization in Section 4 is reasonable, but I might suggest to better signpost the individual sections (see also my comment at the top of Sec 4.1 to understand what I mean). Basically, it would help to give the reader a clearer idea of why you're explaining something to them, since the current organization is rather bottom-up. A more top-down organization can often be easier to follow for this reason, but if you want a bottom-up organization, that can work too, it just requires proper signposting.

In this section, we provide a provably-efficient objective for learning action representations from suboptimal data. Our initial derivations (Theorem~\ref{thm:tabular}) apply to general policies and latent action spaces, while our subsequent result (Theorem~\ref{thm:linear}) provides improved bounds for specialized settings with continuous latent action spaces. Finally, we present our practical method \method for action representation learning and downstream imitation learning.

%and learn a latent policy $\pirep:\Sset\to\Zset$ in the latent action space through behavioral cloning. 
%Learning action representations through modeling transitions also has an intuitive interpretation --- ``rotating clockwise by $90$ degrees" in maze navigation is equivalent to ``rotating counter-clockwise by $270$ degrees" seen from the next state, and is equivalent to ``face east" (i.e., a latent action) when an agent is currently facing north.
%%SL.10.3: I actually don't understand this example at all, sorry :(
%regardless of where an agent is located as long as it is currently facing north. %By learning the latent action space $\Zset$, the combinatorially large input-output space of $\pi:\Sset\to\Aset$ is effectively reduced.

\subsection{Performance Bound with Reparametrized Actions}

%%SL.10.3: I think the organization would be easier to follow if the paragraphs start with a brief sentence or two stating what the goal of the section/subsection is (e.g., In this section, we will show that [something] does [something good]. We will later use this to motivate [something we actually do]). Basically, tell the reader what you're doing here, and how it relates to the overall goals of the paper, so that it's not just a stream of thoughts that leave the reader wondering where you're going with all this.
%As a first step, we want to reparametrize actions from the original action space $\Aset$ into a latent action space $\Zset$. Note that this reparameterization also involves the current state; ``face east" makes sense only in the context of ``currently facing north". Therefore, we use the function $\phi:\Sset\times\Aset\to\Zset$ for this reparameterization. Next, we want $\phi$ to capture aspects of the environment important to our learning goal $\Diff(\pi,\pitarget)$, which involves the transition probabilities. To this end, we introduce a factored transition model $\Trep$ with $\phi$ being a factor, and relate the objective for learning $\Trep$ and $\phi$ to that of imitation learning.

Despite $\Doff$ being highly suboptimal (e.g., with random actions), the large set of $(s,a,s')$ tuples from $\Doff$ reveals the transition dynamics of the environment, which a latent action space should support. Under this motivation, we propose to learn a \emph{factored} transition model $\Tmodel\defeq \Trep \circ \phi$ from the offline dataset $\Doff$, where $\phi:\Sset\times\Aset\to\Zset$ is an action representaiton function and $\Trep:\Sset\times\Zset\to\Delta(\Sset)$ is a latent transition model. Intuitively, good action representations should enable good imitation learning.

We formalize this intuition in the theorem below by establishing a bound on the quality of a learned policy based on (\ref{eq:tabular-rep}) an offline pretraining objective for learning $\phi$ and $\Trep$, (\ref{eq:tabular-dec}) an offline decoding objective for learning an action decoder $\pidec$, and (\ref{eq:tabular-bc}) a downstream imitation learning objective for learning a latent policy $\pi_Z$ with respect to latent actions determined by $\phi$.

\begin{theorem}
\label{thm:tabular}
Consider an action representation function $\phi:\Sset\times\Aset\to\Zset$, a factored transition model $\Trep:\Sset\times\Zset\to\Delta(\Sset)$, an action decoder $\pidec:\Sset\times\Zset\to\Delta(\Aset)$, and a tabular latent policy $\pirep:\Sset\to\Delta(\Zset)$. Define the transition representation error as
\begin{align}
    \jtrans(\Trep, \phi) &\defeq \E_{(s,a)\sim\visitrb}\left[\dkl(\Trans(s,a)\|\Trep(s, \phi(s,a)))\right],\nonumber
\end{align}
the action decoding error as 
\begin{equation}
    \jbcdec(\pidec,\phi) \defeq \E_{(s,a)\sim\visitrb}[-\log\pidec(a|s, \phi(s, a))],\nonumber
\end{equation}
and the latent behavioral cloning error as 
\begin{equation*}
        \jbcrep(\pirep) \defeq \E_{(s,a)\sim(\visittarget,\pitarget)}[-\log\pirep(\phi(s,a)|s)].\nonumber
\end{equation*}
Then the TV divergence between the state visitation distributions of $\pidec\circ\pirep:\Sset\to\Delta(\Aset)$ and $\pitarget$ can be bounded as
%\ofir{The introduction of the max over $z$ doesn't allow us to have the equality you assert in eq 2.}
\begin{minipage}{\textwidth}
\begin{center}
\begin{equation*}
    \hspace{-25em}\Diff(\pidec\circ\pirep,\pitarget) \leq
\end{equation*}
\begin{empheq}[left={\text{Pretraining}\empheqlbrace}]{align}
& C_1 \cdot\sqrtexplained{\frac{1}{2}\underbrace{\E_{(s,a)\sim\visitrb}\left[\dkl(\Trans(s,a)\|\Trep(s, \phi(s,a)))\right]}_{\displaystyle={\color{blue}\jtrans(\Trep, \phi)}}}\label{eq:tabular-rep}
\\
+& C_2 \cdot\sqrtexplained{\frac{1}{2}\underbrace{\E_{s\sim\visitrb}[
\max_{z\in\Zset}
\dkl(\pi_{\alpha^*}(s,z)\|\pidec(s,z))]}_{\displaystyle\approx~\const(\visitrb,\phi) + {\color{blue}\jbcdec(\pidec,\phi)}}} \label{eq:tabular-dec}
\end{empheq}
\begin{empheq}[left={\hspace{-1.8cm}\parbox{1.8cm}{\text{Downstream} \\ \text{Imitation}}\empheqlbrace}]{align}
+& C_3\cdot\sqrtexplained{\frac{1}{2}\underbrace{\E_{s\sim \visittarget}[\dkl(\pi_{*,Z}(s)\|\pirep(s))]}_{\displaystyle =~\const(\pitarget,\phi) + {\color{red}\jbcrep(\pirep)}}},\label{eq:tabular-bc}
\end{empheq}
\end{center}
\end{minipage}

where $C_1 = \gamma|A|(1-\gamma)^{-1}(1+\dchi(\visittarget\|\visitrb)^{\frac{1}{2}})$, $C_2=\gamma(1-\gamma)^{-1}(1+\dchi(\visittarget\|\visitrb)^{\frac{1}{2}})$, $C_3=\gamma(1-\gamma)^{-1}$, $\pi_{\alpha^*}$ is the optimal action decoder for a specific data distribution $\visitrb$ and a specific $\phi$: 
\begin{equation*}
\pi_{\alpha^*}(a|s,z) = \frac{\visitrb(s,a)\cdot\mathbbm{1}[z=\phi(s,a)]}{\sum_{a'\in\Aset }\visitrb(s,a')\cdot\mathbbm{1}[z=\phi(s,a')]},
\end{equation*}
and $\pi_{*,Z}$ is the marginalization of $\pitarget$ onto $\Zset$ according to $\phi$:
\begin{equation*}
    \pi_{*,\Zset}(z|s) \defeq \sum_{a\in\Aset, z=\phi(s,a)}\pitarget(a|s).
\end{equation*}
\end{theorem}
Theorem~\ref{thm:tabular} essentially decomposes the imitation learning error into (\ref{eq:tabular-rep}) a transition-based representation error $\jtrans$, (\ref{eq:tabular-dec}) an action decoding error $\jbcdec$, and (\ref{eq:tabular-bc}) a latent behavioral cloning error $\jbcrep$. Notice that only (\ref{eq:tabular-bc}) requires expert data $\Demos$; (\ref{eq:tabular-rep}) and (\ref{eq:tabular-dec}) are trained on the large offline data $\Doff$. By choosing $|\Zset|$ that is smaller than $|\Aset|$, fewer demonstrations are needed to achieve small error in $\jbcrep$ compared to vanilla BC with $\jbc$. The Pearson $\chi^2$ divergence term $\dchi(\visittarget\|\visitrb)$ in $C_1$ and $C_2$ accounts for the difference in state visitation between the expert and offline data. In the case where $\visittarget$ differs too much from $\visitrb$, known as the distribution shift problem in offline RL~\citep{levine2020offline}, the errors from $\jtrans$ and $\jbcdec$ are amplified and the terms (\ref{eq:tabular-rep}) and (\ref{eq:tabular-dec}) in Theorem~\ref{thm:tabular} dominate. Otherwise, as $\jtrans\rightarrow0$ and $\pidec,\phi\rightarrow\argmin\jbcdec$, optimizing $\pirep$ in the latent action space is guaranteed to optimize $\pi$ in the original action space.

\paragraph{Sample Complexity}
To formalize the intuition that a smaller latent action space $|\Zset|<|\Aset|$ leads to more sample efficient downstream behavioral cloning, we provide the following theorem in the tabular action setting. First, assume access to an oracle latent action representation function $\phi_{\sopt}\defeq\opt_\phi(\Doff)$ which yields pretraining errors (\ref{eq:tabular-rep})$(\phi_{\sopt})$ and (\ref{eq:tabular-dec})$(\phi_{\sopt})$ in Theorem~\ref{thm:tabular}. For downstream behavioral cloning, we consider learning a tabular $\pirep$ on $\Demos$ with $n$ expert samples. We can bound the expected difference between a latent policy $\pi_{\sopt,\Zset}$ with respect to $\phi_{\sopt}$ and $\pitarget$ as follows.
\begin{theorem}
\label{thm:sample}
Let $\phi_{\sopt}\defeq \opt_\phi(\Doff)$ and $\pi_{\sopt,\Zset}$ be the latent BC policy with respect to $\phi_{\sopt}$. We have,
\begin{equation*}
    \E_{\Demos}[\Diff(\pi_{\sopt,\Zset}, \pitarget)] \le (\ref{eq:tabular-rep})(\phi_{\sopt}) + (\ref{eq:tabular-dec})(\phi_{\sopt}) + C_3\cdot\sqrt{\frac{|\Zset||\Sset|}{n}},
\end{equation*}
where $C_3$ is the same as in Theorem~\ref{thm:tabular}.
\end{theorem}
%\begin{proof}
%See Appendix~\ref{app:sample}
%\end{proof}
We can contrast this bound to its form in the vanilla BC setting, for which $|\Zset|=|\Aset|$ and both (\ref{eq:tabular-rep})$(\phi_{\sopt})$ and (\ref{eq:tabular-dec})$(\phi_{\sopt})$ are zero. We can expect an improvement in sample complexity from reparametrized actions when the errors in (\ref{eq:tabular-rep}) and (\ref{eq:tabular-dec}) are small and $|\Zset| < |\Aset|$.

\subsection{Linear Transition Models with Deterministic Latent Policy}\label{sec:linear}
Theorem \ref{thm:tabular} has introduced the notion of a latent expert policy $\pi_{*,Z}$, and minimizes the KL divergence between $\pi_{*,Z}$ and a \emph{tabular} latent policy $\pirep$. However, it is not immediately clear, in the case of continuous actions, how to ensure that the latent policy $\pirep$ is expressive enough to capture any $\pi_{*,Z}$. In this section, we provide guarantees for recovering stochastic expert policies with continuous action space under a linear transition model.

Consider a \emph{continuous} latent space $\Zset\subset\R^d$ and a \emph{deterministic} latent policy $\pilrep(s)=\theta_s$ for some $\theta\in\R^{d\times|S|}$. While a deterministic $\theta$ in general cannot capture a stochastic $\pitarget$, we show that under a linear transition model $\Trep(s'|s,\phi(s,a))=w(s')^\top \phi(s,a)$, there always exists a deterministic policy $\theta:\Sset\to\R^d$, such that $\theta_s = \pi_{*,Z}(s),\,\forall s\in\Sset$. This means that our scheme for offline pretraining paired with downstream imitation learning can \emph{provably} recover any expert policy $\pitarget$ from a deterministic $\pilrep$, regardless of whether $\pitarget$ is stochastic.
\begin{theorem}
\label{thm:linear}
Let $\phi:\Sset\times A \to\Zset$ for some $\Zset\subset\R^d$ and suppose there exist $w:\Sset\to\R^d$ such that $\Trep(s'|s,\phi(s,a))=w(s')^\top \phi(s,a)$ for all $s,s'\in\Sset,a\in\Aset$.
Let $\pidec:\Sset\times\Zset\to\Delta(\Aset)$ be an action decoder, $\pi:\Sset\to\Delta(\Aset)$ be any policy in $\mdp$ and $\pilrep:\Sset\to\R^d$ be a deterministic latent policy for some $\theta\in\R^{d\times|S|}$.
Then,

\begin{minipage}{\textwidth}
\begin{center}
\begin{equation*}
    \hspace{-15em}
    \Diff(\pidec\circ\pilrep,\pitarget) \leq (\ref{eq:tabular-rep})(\Trep,\phi) + (\ref{eq:tabular-dec})(\pidec,\phi)
\end{equation*}
% \begin{empheq}[left={\hspace{-1.3cm}\parbox{2.9cm}{\text{Pretraining} \\ \text{(same as Theorem~\ref{thm:tabular})}}\empheqlbrace}]{align*}
% & C_1 \cdot\sqrtexplained{\frac{1}{2}\underbrace{\E_{(s,a)\sim\visitrb}\left[\dkl(\Trans(s,a)\|\Trep(s, \phi(s,a)))\right]}_{\displaystyle={\color{red}\jtrans(\Trep, \phi)}}}
% \\
% +& C_2 \cdot\sqrtexplained{\frac{1}{2}\underbrace{\E_{(s, a)\sim\visitrb}[\dkl(\pi_{\alpha^*}(s,\phi(s,a))\|\pidec(s,\phi(s,a)))]}_{\displaystyle=~\const(\pi_{\alpha^*},\phi) + {\color{red}\jbcdec(\pidec)}}}
% \end{empheq}
\begin{empheq}[left={\hspace{-2cm}\parbox{1.8cm}{\text{Downstream} \\ \text{Imitation}}\empheqlbrace}]{align}
 &+ C_4 \cdot \left\|\frac{\partial}{\partial\theta} \E_{s\sim d^{\pitarget}, a\sim\pitarget(s)}[(\theta_s - \phi(s,a))^2]\right\|_1,\label{eq:linear-bc}
\end{empheq}
\end{center}
\end{minipage}

where $C_4=\frac{1}{4}|S|\|w\|_\infty$, (\ref{eq:tabular-rep}) and (\ref{eq:tabular-dec}) corresponds to the first and second terms in the bound in Theorem~\ref{thm:tabular}.
\end{theorem}
By replacing term (\ref{eq:tabular-bc}) in Theorem~\ref{thm:tabular} that corresponds to behavioral cloning in the latent action space by term (\ref{eq:linear-bc}) in Theorem~\ref{thm:linear} that is a convex function unbounded in all directions, we are guaranteed that $\pilrep$ is provably optimal regardless of the form of $\pitarget$ and $\pi_{*,Z}$.
Note that the downstream imitation learning objective implied by term (\ref{eq:linear-bc}) is simply the mean squared error between actions $\theta_s$ chosen by $\pi_\theta$ and reparameterized actions $\phi(s,a)$ appearing in the expert dataset.

\subsection{\method: Reparametrized Actions and Imitation Learning in Practice}\label{sec:learning}
%%SL.10.3: I think this section would be easier to follow if it was somewhat more explicit in describing the procedure. Perhaps pseudocode would help

In this section, we describe our learning framework, Transition-Reparametrized Actions for Imitation Learning (\method). \method consists of two training stages: pretraining and downstream behavioral cloning. During pretraining, \method learns $\Trep$ and $\phi$ by minimizing $\jtrans(\Trep,\phi)=\E_{(s,a)\sim\visitrb}\left[\dkl(\Trans(s,a)\|\Trep(s, \phi(s,a)))\right]$. Also during pretraining, \method learns $\pidec$ and $\phi$ by minimizing $\jbcdec(\pidec,\phi) \defeq \E_{(s,a)\sim\visitrb}[-\log\pidec(a|s, \phi(s, a))]$. 
%Note that the $\max_{z\in\Zset}$ in Theorem~\ref{thm:tabular} objective (\ref{eq:tabular-dec}) is intractable, so we approximate the maximum over $\Zset$ with $\E_{(s,a)\sim\visitrb}[\phi(s,a)]$ in practice. 
\method parametrizes $\pidec$ as a multivariate Gaussian distribution. Depending on whether $\Trep$ is defined according to Theorem~\ref{thm:tabular} or Theorem~\ref{thm:linear}, we have either \method EBM or \method linear.

%uses suboptimal offline data to learn $\Trep,\phi,\pi_\alpha$ to minimize terms (\ref{eq:tabular-rep}) and (\ref{eq:tabular-dec}) in Theorem~\ref{thm:tabular} via the offline objectives $\jtrans$ and $\jbcdec$.\ofir{Can we explicitly write out what the pretraining objective is? Just referencing the theorems can be confusing, since in practice we don't care about constants and don't care about square roots. Also, you may want to be detailed in explaining what is being trained for which objective (e.g., for the decoding objective, do we train $\phi$?).} 

%In the text below, we elaborate on specific parameterizations of $\Trep$ to correspond to the bounds in either Theorem~\ref{thm:tabular} or Theorem~\ref{thm:linear}.

%\ofir{The following text should probably appear before we talk specifics about TRAIL EBM -- it should apply to either variant of TRAIL, right?} 
%During pretraining, \method learns $\pidec$ by minimizing $\jbcdec(\pidec,\phi)$ defined in Theorem~\ref{thm:tabular} on $\Doff$. Note that the $\max_{z\in\Zset}$ in (\ref{eq:tabular-dec}) is intractable, so we approximate the maximum over $\Zset$ with $\E_{(s,a)\sim\visitrb}[\phi(s,a)]$ in practice. We allow $\phi$ to be updated while optimizing $\jbcdec$. During downstream imitation learning, \method learns $\pirep$ by minimizing $\jbcrep(\pirep)$ defined in Theorem~\ref{thm:tabular} on $\Demos$.

\textbf{\method EBM for Theorem~\ref{thm:tabular}.} In the tabular action setting that corresponds to Theorem~\ref{thm:tabular}, to ensure that the factored transition model $\Trep$ is flexible to capture any complex (e.g., multi-modal) transitions in the offline dataset, we propose to use an energy-based model (EBM) to parametrize $\Trep(s'|s,\phi(s,a))$,
\begin{equation}
    \Trep(s'|s, \phi(s,a)) \propto \rho(s')\text{exp}(-\|\phi(s,a) - \psi(s')\|^2),
\end{equation}
where $\rho$ is a fixed distribution over $\Sset$. In our implementation we set $\rho$ to be the distribution of $s'$ in $\visitrb$, which enables a practical learning objective for $\Trep$ by minimizing $\E_{(s,a)\sim\visitrb}\left[\dkl(\Trans(s,a)\|\Trep(s, \phi(s,a)))\right]$ in Theorem~\ref{thm:tabular} using a contrastive loss:
\begin{multline*}
    \E_{\visitrb}[-\log\Trep(s'|s, \phi(s, a)))] = \mathrm{const}(\visitrb) + \frac{1}{2}\E_{\visitrb}[||\phi(s, a) - \psi(s')||^2] \\
    + \log\E_{\tilde{s}'\sim\rho}[\exp\{-\frac{1}{2}||\phi(s, a) - \psi(\tilde{s}')||^2\}].
\end{multline*}
During downstream behavioral cloning, \method EBM learns a latent Gaussian policy $\pirep$ by minimizing $\jbcrep(\pirep)=\E_{(s,a)\sim(\visittarget,\pitarget)}[-\log\pirep(\phi(s,a)|s)]$ with $\phi$ fixed. During inference, \method EBM first samples a latent action according to $z\sim\pirep(s)$, and decodes the latent action using $a\sim\pidec(s,z)$ to act in an environment.~\Figref{fig:framework} describes this process pictorially. 

\textbf{\method Linear for Theorem~\ref{thm:linear}.} In the continuous action setting that corresponds to Theorem~\ref{thm:linear}, we propose \method linear, an approximation of \method EBM, to enable learning \emph{linear} transition models required by Theorem~\ref{thm:linear}. Specifically, we first learn $f,g$ that parameterize an energy-based transition model $\Tmodel(s'|s,a)\propto \rho(s')\exp\{-||f(s, a) - g(s')||^2/2\}$ using the same contrastive loss as above (replacing $\phi$ and $\psi$ by $f$ and $g$), and then apply random Fourier features~\citep{rahimi2007random} to recover $\bar\phi(s,a)=\cos(Wf(s,a)+b)$, where $W$ and $b$ are implemented as an untrainable neural network layer on top of $f$. This results in an approximate linear transition model,
\begin{equation*}
\Tmodel(s'|s,a)\propto\rho(s')\exp\{-||f(s,a) - g(s')||^2/2\} \propto \bar\psi(s')^\top\bar\phi(s,a). 
\end{equation*}
During downstream behavioral cloning, \method linear learns a deterministic policy $\pilrep$ in the continuous latent action space determined by $\bar{\phi}$ via minimizing $\left\|\frac{\partial}{\partial\theta} \E_{s\sim d^{\pitarget}, a\sim\pitarget(s)}[(\theta_s - \bar\phi(s,a))^2]\right\|_1$ with $\bar\phi$ fixed.
%the term (\ref{eq:linear-bc}) in Theorem~\ref{thm:linear}. The parametrization and learning objective of $\pidec$ follow those of \method EBM.
During inference, \method linear first determines the latent action according to $z = \pilrep(s)$, and decodes the latent action using $a\sim\pidec(s,z)$ to act in an environment.

%% file: experiment.tex
\section{Experimental Evaluation}
\begin{figure}[b]
\centering
 \includegraphics[width=\linewidth]{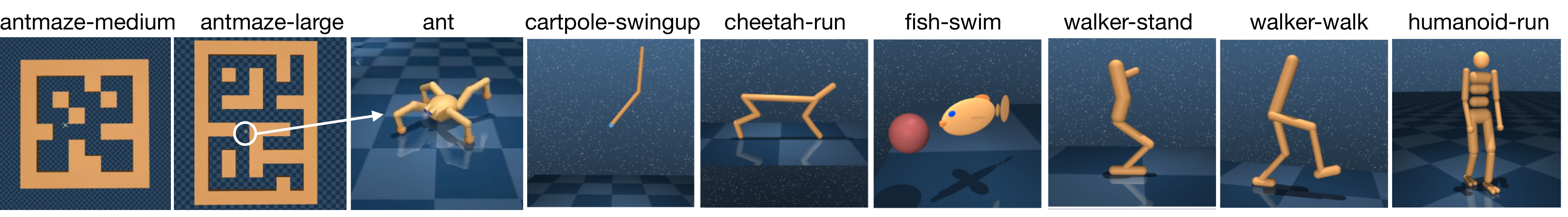}
 \caption{Tasks for our empirical evaluation. We include the challenging AntMaze navigation tasks from D4RL~\citep{fu2020d4rl} and low (1-DoF) to high (21-DoF) dimensional locomotaion tasks from DeepMind Control Suite~\citep{tassa2018deepmind}.}
 \label{fig:tasks}
\end{figure}

We now evaluate \method on a set of navigation and locomotion tasks (\Figref{fig:tasks}). 
%Our proposed method, \method, achieves more significant improvements over vanilla BC compared to temporal skill extraction methods, and performs competitively compared to an offline RL method while using orders of magnitude less expert data.
%%SL.10.3: My suggestions here would be: (1) don't mention the offline RL stuff just yet, it kind of sounds like a type error; bring this up later when you mention the rlu experiments, but it needs the right context to be introduced; (2) scope the claims more about using suboptimal data to boost BC first, and only second as comparing to skill extraction methods (there also needs to be some justification). Here is a potential rephrasing of the claims:
Our evaluation is designed to study how well \method can improve imitation learning with limited expert data by leveraging available suboptimal offline data. We evaluate the improvement attained by \method over vanilla BC, and additionally compare \method to previously proposed temporal skill extraction methods. Since there is no existing benchmark for imitation learning with suboptimal offline data, we adapt existing datasets for offline RL, which contain suboptimal data, and augment them with a small amount of expert data for downstream imitation learning.

%%SL.10.3: One problem with the current organization of the experiments section is that it is organized by task rather than by research hypothesis. I don't think that's great, because it results in a lot of redundancy, makes the experiments section longer than it should be, and makes it harder to get at the core message. Perhaps a better organization would focus more on the conclusions. So you could have one subsection that covers all comparisons (both d4rl and rlu), and then a separate subsection that covers ablations and various analysis.

\subsection{Evaluating Navigation without Temporal Abstraction}
%%SL.10.3: Well, presumably the point of the paper is *not* temporal abstraction. It's generally a good idea to position the results in terms of your method being good, rather than prior methods being bad, so perhaps rename it "Comparative Evaluation with Suboptimal Data" or something, rather than positioning it entirely in opposition to prior papers.

\begin{figure}[t]
\centering
%%SL.10.3: it would be good to slightly alter the labels to indicate both the offline dataset and the expert dataset (I realize expert data will be the same in all cases, but this might address some confusion)5\%
 \includegraphics[width=\linewidth]{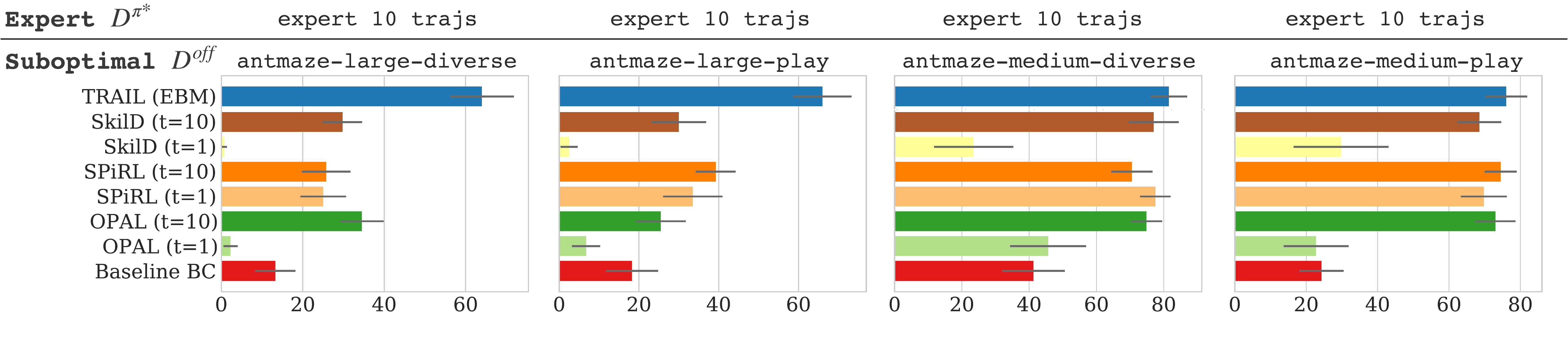}
 \caption{Average success rate ($\%$) over $4$ seeds of \method EBM (Theorem~\ref{thm:tabular}) and temporal skill extraction methods -- SkiLD~\citep{pertsch2021guided}, SPiRL~\citep{pertsch2020accelerating}, and OPAL~\citep{ajay2020opal} -- pretrained on suboptimal $\Doff$. Baseline BC corresponds to direct behavioral cloning of expert $\Demos$ without latent actions.}
 \label{fig:antmaze}
\end{figure}

\paragraph{Description and Baselines.} We start our evaluation on the AntMaze task from D4RL~\citep{fu2020d4rl}, which has been used as a testbed by recent works on temporal skill extraction for few-shot imitation~\citep{ajay2020opal} and RL~\citep{ajay2020opal,pertsch2020accelerating,pertsch2021guided}. 
%%SL.10.3: Agreed, we really need more pictures! pictures of the tasks would really help, just a big banner picture across the top showing all the tasks somewhere
%%SL.10.3: One potential issue here (though probably this should be addressed in related work, not here) is that lots of hierarchical RL researchers who don't work on offline RL will get very unhappy right here, since they will perceive you as ignoring the decades of work on hierarchical RL and just citing three very recent offline RL papers that are not mainstream in the HRL community. Be conscious of this, right now the paper as a whole will come off as rather objectionable to traditional HRL researchers because it doesn't acknowledge their contributions enough. One thing you could do for example is, after the next sentence, add something that says: While a number of other works have also proposed to learn primitives for hierarchical imitation and RL~\citep{tons_of_citations}, we chose the above methods for comparison because they are (i) recent and (ii) specifically designed for handling offline suboptimal data`. [but really you should check if this is true, eg some of the option discovery papers from Ken Goldberg's group are also aimed precisely at extracting skills from demo data, and there are probably many others, this is a very old field]
We compare \method to OPAL~\citep{ajay2020opal}, SkilD~\citep{pertsch2021guided}, and SPiRL~\citep{pertsch2020accelerating}, all of which use an offline dataset to extract temporally extended (length $t=10$) skills to form a latent action space for downstream learning. SkiLD and SPiRL are originally designed only for downstream RL, so we modify them to support downstream imitation learning as described in Appendix~\ref{app:exp}. While a number of other works have also proposed to learn primitives for hierarchical imitation~\citep{kipf2019compile,hakhamaneshi2021hierarchical} and RL~\citep{fox2017multi,krishnan2017ddco,shankar2019discovering,shankar2020learning,singh2020parrot}, we chose OPAL, SkiLD, and SPiRL for comparison because they are the most recent works in this area with reported results that suggest these methods are state-of-the-art, especially in learning from \emph{suboptimal} offline data based on D4RL.
%%SL.10.3: should make it clear which methods were designed for hrl rather than imitation too, to avoid misunderstandings
To construct the suboptimal and expert datasets,
%%SL.10.3: Perhaps it will be clearer to call these "expert" datasets rather than "target" datasets? That might be a little bit clearer. In terms "offline dataset" -- I think we should consider renaming this, as *both* datasets are offline, so this is not particularly discriminative. Maybe "background" vs "expert" or "suboptimal" vs "expert" is better?
we follow the protocol in~\citet{ajay2020opal}, which uses the full \texttt{diverse} or \texttt{play} D4RL AntMaze datasets as the suboptimal offline data, while using a set of $n=10$ expert trajectories (navigating from one corner of the maze to the opposite corner) as the expert data. The \texttt{diverse} and \texttt{play} datasets are suboptimal in the corner-to-corner navigation task, as they only contain data that navigates to random or fixed locations different from task evaluation. %See further experimental details in Appendix~\ref{app:exp}. 
%%SL.10.3: It might help to add a little bit more detail about the way in which the offline datasets are suboptimal.

%%SL.10.3: Perhaps add a \paragraph heading to indicate that this is implementation details or something? In general we could organize this section better
\paragraph{Implementation Details.} For \method, we parameterize $\phi(s,a)$ and $\psi(s')$ using separate feed-forward neural networks (see details in Appendix~\ref{app:exp}) and train the transition EBM via the contrastive objective described in~\Secref{sec:learning}. We parametrize both the action decoder $\pidec$ and the latent $\pirep$ using multivariate Gaussian distributions with neural-network approximated mean and variance. For the temporal skill extraction methods, we implement the trajectory encoder using a bidirectional RNN and parametrize skill prior, latent policy, and action decoder as Gaussians following~\citet{ajay2020opal}. We adapt SPiRL and SkiLD
%%SL.10.3: , which were designed for hierarchical RL rather than imitation learning,
for imitation learning by including the KL Divergence term between the latent policy and the skill prior during downstream behavioral cloning (see details in Appendix~\ref{app:exp}). We do a search on the extend of temporal abstraction, and found $t=10$ to work the best as reported in these papers' maze experiments. We also experimented with a version of vanilla BC pretrained on the suboptimal data and fine-tuned on expert data for fair comparison, which did not show a significant difference from directly training vanilla BC on expert data.

\paragraph{Results.} \Figref{fig:antmaze} shows the average performance of \method in terms of task success rate (out of 100\%) compared to the prior methods. Since all of the prior methods are proposed in terms of temporal abstraction, we evaluate them both with the default temporal abstract, $t=10$, as well as without temporal abstraction, corresponding to $t=1$.
%%SL.10.3: should that be t=1? why is it it 2?
Note that \method uses \emph{no} temporal abstraction. 
We find that on the simpler \texttt{antmaze-medium} task, \method trained on a
%%SL.10.3: is the word "from" here a typo?
a single-step transition model performs similarly to the set of temporal skill extraction methods with $t=10$. However, these skill extraction methods experience a degradation in performance when temporal abstraction is removed ($t=1$). This corroborates the existing theory in these works~\citep{ajay2020opal}, which attributes their benefits predominantly to temporal abstraction rather than producing a latent action space that is ``easier'' to learn. Meanwhile, \method is able to excel without any temporal abstraction. 

These differences become even more pronounced on the harder \texttt{antmaze-large} tasks. We see that \method maintains significant improvements over vanilla BC, whereas temporal skill extraction fails to achieve good performance even with $t=10$.
%%SL.10.3: instead of presenting prior work as doing something negative, present your work as doing something positive, i.e.: Our method outperforms... rather than Prior methods suck... It just predisposes your reader more positively toward you.
%%SL.10.3: could we skip the footnote and just add the bar to these plots? It would not take much space, and it would only add one sentence to explain it, which would save us an unnecessary trip to the appendix and a disruptive footnote (and we have the space for it).
%Overall, our results suggest that the benefit of temporal abstraction is a result of the specific action representation learning objective chosen, as opposed to navigational tasks by nature requiring temporal abstraction. We do emphasize, however, that this is not equivalent to saying that temporal abstraction is never helpful, but rather there is much to gain from learning action representations on single-steps alone.
%%SL.10.3: I generally agree with the spirit of the above two sentences, but they are a little hard to parse. Could rephrase like this: 
These results suggest that \method attains significant improvement specifically from utilizing the suboptimal data for learning suitable action representations, rather than simply from providing temporal abstraction. Of course, this does not mean that temporal abstraction is never helpful. Rather, our results serve as evidence that suboptimal data can be useful for imitation learning not just by providing temporally extended skills, but by actually reformulating the action space to make imitation learning easier and more efficient.

\subsection{Evaluating Locomotion with Highly Suboptimal Offline Data}
%%SL.10.3: It's not clear to me what "degradation" is referring to here.

\begin{figure}[t]
\centering
%%SL.10.3: it would be good to slightly alter the labels to indicate both the offline dataset and the expert dataset (I realize expert data will be the same in all cases, but this might address some confusion)
 \includegraphics[width=\linewidth]{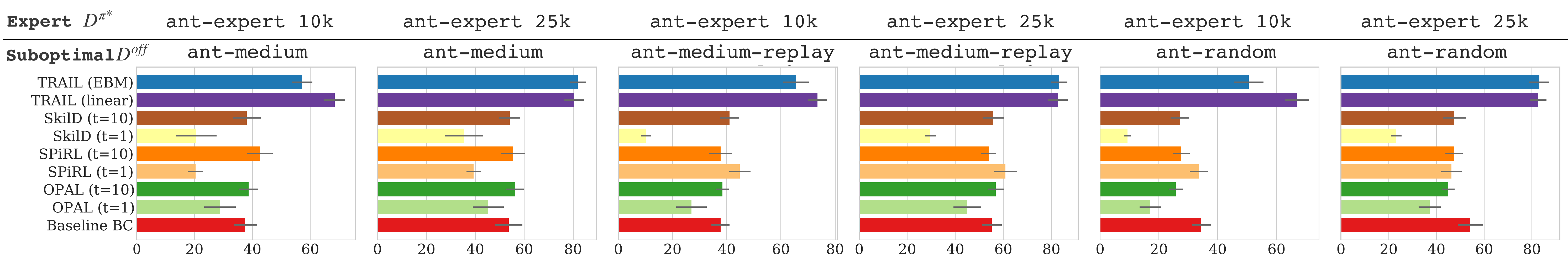}
\caption{Average rewards (over $4$ seeds) of \method EBM (Theorem~\ref{thm:tabular}), \method linear (Theorem~\ref{thm:linear}), and baseline methods when using a variety of unimodal (\texttt{ant-medium}), low-quality (\texttt{ant-medium-replay}), and random (\texttt{ant-random}) offline datasets $\Doff$ paired with a smaller expert dataset $\Demos$ (either $10$k or $25$k expert transitions).}
 \label{fig:ant}
\end{figure}
%%SL.10.3: it's a bit strange that this figure presents learning curves while Fig 1 and Fig 2 are bar graphs

\paragraph{Description.}The performance of \method trained on a \emph{single-step} transition model in the previous section suggests that learning single-step latent action representations can benefit a broader set of tasks for which temporal abstraction may not be helpful, e.g., when the offline data is highly suboptimal (with near-random actions) or unimodal (collected by a single stationary policy).
%%SL.10.3: Same comment as before about positioning it too much "in opposition" to prior work. The paper is not about temporal abstraction, I think you've made your point already. Let's position it in terms of this method being *good*, not in terms of prior methods being *bad*.
In this section, we
%%SL.10.3: It's not clear to what the "therefore" is doing here -- how does the stuff below follow from the stuff above?
consider a Gym-MuJoCo task from D4RL using the same 8-DoF quadruped ant robot as the previously evaluated navigation task. We first learn action representations from the \texttt{medium}, \texttt{medium-replay}, or \texttt{random} datasets, and imitate from $1\%$ or $2.5\%$ of the \texttt{expert} datasets from D4RL. The \texttt{medium} dataset represents data collected from a mediocre stationary policy (exhibiting unimodal behavior), and the \texttt{random} dataset is collected by a randomly initialized policy and is hence highly suboptimal.
%%SL.10.3: In general, I would suggest one of two things here: (1) either combine these results with the Sec 5.1 results in one section and present them together (and probably combine with Sec 5.3), which I think is the better choice anyway because it consolidates all the comparisons in one place, or (2) better motivate why we should also compare on these tasks.

\paragraph{Implementation Details.} For this task, we additionally train a linear version of \method by approximating the transition EBM using random Fourier features~\citep{rahimi2007random} and learn a \emph{deterministic} latent policy following Theorem~\ref{thm:linear}. Specifically, we use separate feed-forward networks to parameterize $f(s, a)$ and $g(s')$, and extract action representations using $\phi(s,a) = \cos(W f(s,a) + b)$, where $W,b$ are untrainable randomly initialized variables as described in~\Secref{sec:learning}. Different from \method EBM which parametrizes $\pirep$ as a Gaussian, \method linear parametrizes the \emph{deterministic} $\pilrep$ using a feed-forward neural network.
%%SL.10.3: why?

\paragraph{Results.} Our results are shown in \Figref{fig:ant}. Both the EBM
%%SL.10.3: It's a bit strange to use "EBM" as the antonym of "linear" -- the two are not really opposites of one another. While I understand what you mean and why this kind of makes sense, maybe there is a more informative term that could be used?
and linear versions of \method consistently improve over baseline BC, whereas temporal skill extraction methods generally lead to worse performance regardless of the extent of abstraction, likely due to the degenerate effect (i.e., latent skills being ignored by a flexible action decoder) resulted from unimodal offline datasets as discussed in~\citep{ajay2020opal}.
%%SL.10.3: it's not clear what "degenerate effect" this is referring to
Surprisingly, \method achieves a significant performance boost even when latent actions are learned from the \texttt{random} dataset, suggesting the benefit of learning action representations from transition models when the offline data is highly suboptimal.
%%SL.10.3: that seems like a strong claim -- I'm not sure the results show that *learning a dynamics model* is necessary the essential piece
Additionally, the linear variant of \method performs slightly better than the EBM variant when the expert sample size is small (i.e., $10$k), suggesting
%%SL.10.3: ...which agrees with the sample-efficiency benefit suggested by Theorem...
the benefit of learning deterministic latent policies from Theorem~\ref{thm:linear} when the environment is effectively approximated by a linear transition model. %We also evaluated the linear variant of \method with a Gaussian (as opposed to deterministic) latent policy, but noticed little difference in performance, suggesting that learning a deterministic policy is sufficient under a linear approximate transition model.

%\subsection{Comparison to Offline Reinforcement Learning}
\subsection{Evaluation on DeepMind Control Suite}

\begin{figure}[t]
\centering
 \includegraphics[width=\linewidth]{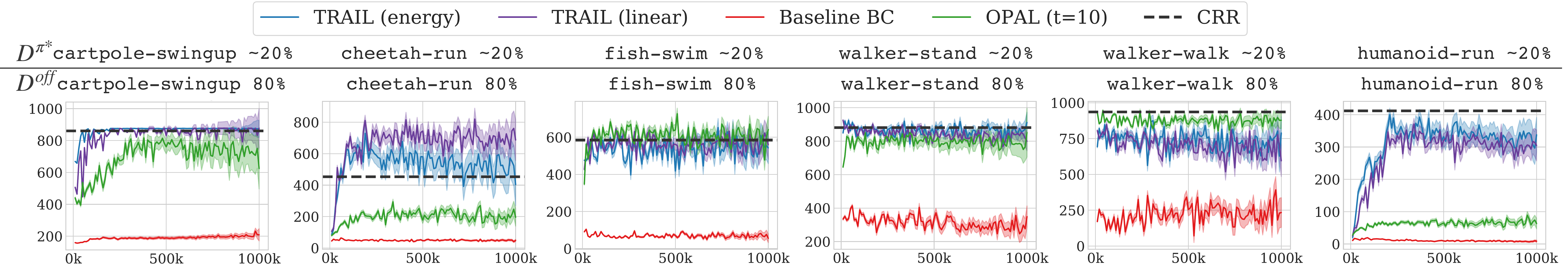}
 \caption{Average task rewards (over $4$ seeds) of \method EBM (Theorem~\ref{thm:tabular}), \method linear (Theorem~\ref{thm:linear}), and OPAL (other temporal methods are included in Appendix~\ref{app:results}) pretrained on the bottom $80\%$ of the RL Unplugged datasets followed by behavioral cloning in the latent action space on $\frac{1}{10}$ of the top $20\%$ of the RL Unplugged datasets following the setup in~\citet{zolna2020offline}. Baseline BC achieves low rewards due to the small expert sample size. Dotted lines denote the performance of CRR~\citep{wang2020critic}, an offline RL method trained on the full RL Unplugged datasets with reward labels.}
 \label{fig:rlu}
\end{figure}

\paragraph{Description.} Having witnessed the improvement \method brings to behavioral cloning on AntMaze and MuJoCo Ant, we wonder how \method perform on a wider spectrum of locomotion tasks with various degrees of freedom.
%compares to running offline RL on the combined offline and expert datasets with reward labels.
%%SL.10.3: see above, maybe not a good idea to prime the reader in this way (because we don't compare to the modern effective offline RL methods, just a relatively weak baseline that no one uses)
We consider $6$ locomotion tasks from the DeepMind Control Suite~\citep{tassa2018deepmind} ranging from simple (e.g., 1-DoF \texttt{cartople-swingup}) to complex (e.g., 21-DoF \texttt{humanoid-run}) tasks. Following the setup in~\citet{zolna2020offline}, we take $\frac{1}{10}$ of the trajectories whose episodic reward is among the top $20\%$ of the open source RL Unplugged datasets~\citep{gulcehre2020rl} as expert demonstrations (see numbers of expert trajectories in Appendix~\ref{app:exp}), and the bottom $80\%$ of RL Unplugged as the suboptimal offline data. For completeness, we additionally include comparison to Critic Regularized Regression (CRR)~\citep{wang2020critic}, an offline RL method with competitive performance on these tasks.
%(which outperforms other offline RL methods such as BCQ~\citep{fujimoto2019off} and BRAC~\citep{wu2019behavior} according to~\citet{wang2020critic,gulcehre2020rl}). 
CRR is trained on the full RL Unplugged datasets (i.e., combined suboptimal and expert datasets) with reward labels.

\paragraph{Results.}\Figref{fig:rlu} shows the comparison results. \method outperforms temporal extraction methods on both low-dimensional (e.g., \texttt{cartpole-swingup}) and high-dimensional (\texttt{humanoid-run}) tasks. Additionally, \method performs similarly to or better than CRR on $4$ out of the $6$ tasks despite not using any reward labels, and only slightly worse on \texttt{humanoid-run} and \texttt{walker-walk}. To test the robustness of \method when the offline data is highly suboptimal, we further reduce the size and quality of the offline data to the bottom $5\%$ of the original RL Unplugged datasets. As shown in~\Figref{fig:rlu5} in Appndix~\ref{app:results}, the performance of temporal skill extraction declines in \texttt{fish-swim}, \texttt{walker-stand}, and \texttt{walker-walk} due to this change in offline data quality, whereas \method maintains the same performance as when the bottom $80\%$ data was used, suggesting that \method is more robust to low-quality offline data. 

This set of results suggests a promising direction for offline learning of sequential decision making policies, namely to learn latent actions from abundant low-quality data and behavioral cloning in the latent action space on scarce high-quality data. Notably, compared to offline RL, this approach is applicable to settings where data quality cannot be easily expressed through a scalar reward.
%%SL.10.3: I think it's also worth remarking that this approach is not mutually exclusive with RL -- that is, a promising direction for future work is to just use exactly the same representation but with RL

%% file: conclusion.tex
\section{Conclusion}
We have derived a near-optimal objective for learning a latent action space from suboptimal offline data that provably accelerates downstream imitation learning. To learn this objective in practice, we propose transition-reparametrized actions for imitation learning (TRAIL), a two-stage framework that first pretrains a factored transition model from offline data, and then uses the transition model to reparametrize the action space prior to behavioral cloning. Our empirical results suggest that TRAIL can improve imitation learning drastically, even when pretrained on highly suboptimal data (e.g., data from a random policy), providing a new approach to imitation learning through a combination of pretraining on task-agnostic or suboptimal data and behavioral cloning on limited expert datasets. That said, our approach to action representation learning is not necessarily specific to imitation learning, and insofar as the reparameterized action space simplifies downstream control problems, it could also be combined with reinforcement learning in future work. More broadly, studying how learned action reparameterization can accelerate various facets of learning-based control represents an exciting future direction, and we hope that our results provide initial evidence of such a potential.
%How to combine action reparametrizations with lower-dimensional state representations to further reduce sample complexity of both imitation learning and RL is an interesting direction of future work.

%% file: appendix.tex
\clearpage
\begin{center}
{\huge Appendix}
\end{center}
\section{Proofs for Foundational Lemmas}
\begin{lemma}
\label{lem:performance}
If $\pi_1$ and $\pi_2$ are two policies in $\mdp$ and $d^{\pi_1}(s)$ and $d^{\pi_2}(s)$ are the state visitation distributions induced by policy $\pi_1$ and $\pi_2$ where $\visitpi(s) := (1-\gamma)\sum_{t=0}\gamma^t\cdot\Pr\left[s_t=s|\pi,\mdp\right]$. Define $\Diff(\pi_2,\pi_1) = \dtv(d^{\pi_2}\|d^{\pi_1})$ then 
\begin{equation}
\label{eq:perf-diff-bound}
\Diff(\pi_2, \pi_1) \leq \frac{\gamma}{1-\gamma}
\err_{d^{\pi_1}}(\pi_1,\pi_2,\Trans),
\end{equation}
where 
\begin{equation}
    \err_{d^{\pi_1}}(\pi_1,\pi_2,\Trans) \defeq \frac{1}{2}\sum_{s'\in\Sset} \left|\E_{s\sim d^{\pi_1},a_1\sim\pi_1(s),a_2\sim\pi_2(s)}[\Trans(s'|s,a_1) - \Trans(s'|s,a_2)]\right|.
\end{equation}
is the TV-divergence between $\Trans\circ\pi_1\circ d^{\pi_1}$ and $\Trans\circ\pi_2\circ d^{\pi_1}$.
\end{lemma}
\begin{proof}
Following similar derivations in~\cite{achiam2017constrained,nachum2018near}, we express $\dtv(d^{\pi_2}\|d^{\pi_1})$ in linear operator notation: %\ofir{I think you are missing a $(1-\gamma)$ factor in the below eqn}
\begin{equation}
    \Diff(\pi_2,\pi_1) = \dtv(d^{\pi_2}\|d^{\pi_1}) = \frac{1}{2}\mathbf{1}|(1-\gamma)(I - \gamma \Trans\Pi_2)^{-1}\init - (1-\gamma)(I - \gamma \Trans\Pi_1)^{-1}\init|,
\end{equation}
where $\Pi_1,\Pi_2$ are linear operators $\Sset\to\Sset\times\Aset$ such that $\Pi_i \nu(s,a) = \pi_i(a|s)\nu(s)$ and $\mathbf{1}$ is an all ones row vector of size $|S|$.
Notice that $d^{\pi_1}$ may be expressed in this notation as $(1-\gamma)(I - \gamma \Trans\Pi_1)^{-1}\init$. 
We may re-write the above term as
\begin{align}
    &\frac{1}{2}\mathbf{1}|(1-\gamma)(I - \gamma \Trans\Pi_2)^{-1}((I - \gamma\Trans\Pi_1) - (I - \gamma\Trans\Pi_2))(I - \gamma \Trans\Pi_1)^{-1}\init| \nonumber\\
    =& \gamma\cdot\frac{1}{2}\mathbf{1}|(I - \gamma \Trans\Pi_2)^{-1}(\Trans\Pi_2 - \Trans\Pi_1) d^{\pi_1}|.
\end{align}
Using matrix norm inequalities, we bound the above by
\begin{equation}
\gamma\cdot\frac{1}{2} \|(I - \gamma \Trans\Pi_2)^{-1}\|_{1,\infty}\cdot \mathbf{1}|(\Trans\Pi_2 - \Trans\Pi_1) d^{\pi_1}|.
\end{equation}
Since $\Trans\Pi_2$ is a stochastic matrix, $\|(I - \gamma \Trans\Pi_2)^{-1}\|_{1,\infty} \le \sum_{t=0}^\infty \gamma^t\|\Trans\Pi_2\|_{1,\infty} = (1-\gamma)^{-1}$. Thus, we bound the above by
\begin{equation}
\frac{\gamma}{2(1-\gamma)}\mathbf{1}|(\Trans\Pi_2 - \Trans\Pi_1) d^{\pi_1}| = \frac{\gamma}{1-\gamma} \err_{d^{\pi_1}}(\pi_1,\pi_2,\Trans),
\end{equation}
and so we immediately achieve the desired bound in~\eqref{eq:perf-diff-bound}.
\end{proof}

The divergence bound above relies on the true transition model $\Trans$ which is not available to us. We now introduce an approximate transition model $\Tmodel$ to proxy $\err_{d^{\pi_1}}(\pi_1,\pi_2,\Trans)$.

\begin{lemma}
\label{lem:model1}
For $\pi_1$ and $\pi_2$ two policies in $\mdp$ and any transition model $\Tmodel(\cdot|s, a)$ we have, 
\begin{align}
\err_{d^{\pi_1}}(\pi_1,\pi_2,\Trans) &\le |\Aset|\E_{(s,a)\sim (d^{\pi_1}, \uniform)}[\dtv(\Trans(s,a)\|\Tmodel(s,a))] + 
\err_{d^{\pi_1}}(\pi_1,\pi_2,\Tmodel).
\end{align}
\end{lemma}
\begin{proof}
\begin{align}
    \err_{d^{\pi_1}}(\pi_1,\pi_2,\Trans) &= \frac{1}{2}\sum_{s'\in\Sset}\left|\E_{s\sim d^{\pi_1},a_1\sim\pi_1(s),a_2\sim\pi_2(s)}[\Trans(s'|s,a_1) - \Trans(s'|s,a_2)]\right|
\end{align}
\begin{align}
    &= \frac{1}{2}\sum_{s'\in\Sset}\left|\sum_{a\in\Aset}\E_{s\sim d^{\pi_1}}[\Trans(s'|s,a)\pi_1(a|s) - \Trans(s,a)\pi_2(a|s)]\right| \\
    &= \frac{1}{2}\sum_{s'\in\Sset}\left|\sum_{a\in\Aset}\E_{s\sim d^{\pi_1}}[(\Trans(s'|s,a) - \Tmodel(s'|s,a))(\pi_1(a|s) - \pi_2(a|s)) + \Tmodel(s'|s,a)(\pi_1(a|s) - \pi_2(a|s))]\right| \\
    &\le \frac{1}{2}\sum_{s'\in\Sset}\left|\sum_{a\in\Aset}\E_{s\sim d^{\pi_1}}[(\Trans(s'|s,a) - \Tmodel(s'|s,a))(\pi_1(a|s) - \pi_2(a|s))]\right| + \err_{d^{\pi_1}}(\pi_1,\pi_2,\Tmodel) \\
    &\le \frac{1}{2}\sum_{s'\in\Sset}\sum_{a\in\Aset}\E_{s\sim d^{\pi_1}}[\left|(\Trans(s'|s,a) - \Tmodel(s'|s,a))(\pi_1(a|s) - \pi_2(a|s))\right|] + \err_{d^{\pi_1}}(\pi_1,\pi_2,\Tmodel) \\
    &\le |\Aset|\E_{(s,a)\sim (d^{\pi_1}, \uniform)}[\dtv(\Trans(s'|s,a)\|\Tmodel(s'|s,a)|] + \err_{d^{\pi_1}}(\pi_1,\pi_2,\Tmodel),
\end{align}
and we arrive at the inequality as desired where the last step comes from $\dtv(\Trans(s,a)\|\Tmodel(s,a)) = \frac{1}{2}\sum_{s'\in\Sset}|\Trans(s'|s,a)-\Tmodel(s'|s,a)|$.
\end{proof}

Now we introduce a representation function $\phi:\Sset\times\Aset\to\Zset$ and show how the error above may be reduced when $\Tmodel(s,a)=\Trep(s, \phi(s, a))$:
\begin{lemma}
\label{lem:bottleneck}
Let $\phi:\Sset\times\Aset\to\Zset$ for some space $\Zset$ and suppose there exists $\Trep:\Sset\times\Zset\to\Delta(\Sset)$ such that $\Tmodel(s,a)=\Trep(s,\phi(s,a))$ for all $s\in\Sset,a\in\Aset$.
Then for any policies $\pi_1,\pi_2$,
\begin{align}
    \err_{d^{\pi_1}}(\pi_1,\pi_2,\Tmodel)] &\le \E_{s\sim d^{\pi_1}}[\dtv(\pi_{1,\Zset}\|\pi_{2,\Zset})],
\end{align}
where $\pi_{k,\Zset}(z|s)$ is the marginalization of $\pi_k$ onto $Z$:
\begin{align}
    \pi_{k,\Zset}(z|s) &\defeq \sum_{a\in\Aset, z=\phi(s,a)}\pi_k(a|s)
\end{align}
for all $z\in\Zset, k\in\{1,2\}$.
\end{lemma}

\begin{proof}
\begin{align}
    & \frac{1}{2}\sum_{s'\in\Sset}\left|\E_{s\sim d^{\pi_1},a_1\sim\pi_1(s),a_2\sim\pi_2(s)}[\Tmodel(s'|s,a_1) - \Tmodel(s'|s,a_2)]\right| \\
    &= \frac{1}{2}\sum_{s'\in\Sset}\left|\sum_{s\in\Sset,a\in\Aset}\Trep(s'|s,\phi(s,a))\pi_1(a|s)d^{\pi_1}(s) - \sum_{s\in\Sset,a\in\Aset}\Trep(s'|s,\phi(s,a))\pi_2(a|s)d^{\pi_1}(s)\right| \nonumber \\
    &=\frac{1}{2}\sum_{s'\in\Sset}\left|\sum_{s\in\Sset,z\in\Zset}\Trep(s'|s,z)\!\!\!\sum_{\substack{a\in\Aset,\\\phi(s,a)=z}}\pi_1(a|s)d^{\pi_1}(s) - \sum_{s\in\Sset,z\in\Zset}\Trep(s'|s,z)\sum_{\substack{a\in\Aset,\\\phi(s,a)=z}}\pi_2(a|s)d^{\pi_1}(s)\right| \nonumber \\
    &=\frac{1}{2}\sum_{s'\in\Sset}\left|\sum_{s\in\Sset,z\in\Zset}\Trep(s'|s,z)\pi_{1,\Zset}(z|s) d^{\pi_1}(s) - \sum_{s\in\Sset,z\in\Zset}\Trep(s'|s,z)\pi_{2,\Zset}(z|s)d^{\pi_1}(s)\right| \nonumber \\
    &=\frac{1}{2}\sum_{s'\in\Sset}\left|\E_{s\sim d^{\pi_1}}\left[\sum_{z\in\Zset}\Trep(s'|s,z)(\pi_{1,\Zset}(z|s) - \pi_{2,\Zset}(z|s))\right]\right| \\
    %&\le \frac{1}{2}\sum_{s'\in\Sset}\sum_{z\in\Zset}\E_{s\sim d^{\pi_1}}\left[\left|\Trep(s'|s,z)(\pi_{1,\Zset}(z|s) - \pi_{2,\Zset}(z|s))\right|\right] \\
    &\le \frac{1}{2}\E_{s\sim d^{\pi_1}}\left[\sum_{z\in\Zset}\sum_{s'\in\Sset}\Trep(s'|s,z)\left|\pi_{1,\Zset}(z|s) - \pi_{2,\Zset}(z|s)\right|\right] \\
    &= \frac{1}{2}\E_{s\sim d^{\pi_1}}\left[\sum_{z\in\Zset}\left|\pi_{1,\Zset}(z|s) - \pi_{2,\Zset}(z|s)\right|\right] \\
    &=\E_{s\sim d^{\pi_1}}\left[\dtv(\pi_{1,\Zset}\|\pi_{2,\Zset})\right],
\end{align}
and we arrive at the inequality as desired. %\ofir{I'm not following eqns 23-25. Can you check these?}
\end{proof}

\begin{lemma}\label{lem:marginal-opt}
Let $d\in\Delta(\Sset,\Aset)$ be some state-action distribution, $\phi:\Sset\times\Aset\to \Zset$, and $\pirep:\Sset\to\Delta(\Zset)$.
Denote $\pi_{\alpha^*}$ as the optimal action decoder for $d,\phi$: 
\begin{equation*}
\pi_{\alpha^*}(a|s,z) = \frac{d(s,a)\cdot\mathbbm{1}[z=\phi(s,a)]}{\sum_{a'\in\Aset }d(s,a')\cdot\mathbbm{1}[z=\phi(s,a')]},
\end{equation*}
and $\pi_{\alpha^*,\Zset}$ as the marginalization of $\pi_{\alpha^*}\circ\pirep$ onto $\Zset$:
\begin{equation*}
    \pi_{\alpha^*,\Zset}(z|s) \defeq \sum_{a\in\Aset, z=\phi(s,a)} (\pi_{\alpha^*}\circ\pirep)(a|s) = \sum_{a\in\Aset, z=\phi(s,a)}\sum_{\tilde z\in\Zset}\pi_{\alpha^*}(a|s, \tilde z)\pirep(\tilde z|s).
\end{equation*}
Then we have
\begin{equation}
    \pi_{\alpha^*,\Zset}(z|s) = \pirep(z|s)
\end{equation}
for all $z\in\Zset$ and $s\in\Sset$.
\begin{proof}
\begin{align}
     \pi_{\alpha^*,\Zset}(z|s) 
     &= \sum_{a\in\Aset, z=\phi(s,a)}\sum_{\tilde z\in\Zset}\pi_{\alpha^*}(a|s, \tilde z)\pirep(\tilde z|s) \\
     &= \sum_{a\in\Aset, z=\phi(s,a)}\sum_{\tilde z\in\Zset}\frac{d(s,a)\cdot\mathbbm{1}[\tilde z=\phi(s,a)]}{\sum_{a'\in\Aset }d(s,a')\cdot\mathbbm{1}[\tilde z=\phi(s,a')]}\pirep(\tilde z|s) \\
     &= \sum_{a\in\Aset, z=\phi(s,a)}\frac{d(s,a)\cdot\mathbbm{1}[ z=\phi(s,a)]}{\sum_{a'\in\Aset }d(s,a')\cdot\mathbbm{1}[ z=\phi(s,a')]}\pirep(z|s) \\
     &= \pirep(z|s) \sum_{a\in\Aset, z=\phi(s,a)}\frac{d(s,a)\cdot\mathbbm{1}[ z=\phi(s,a)]}{\sum_{a'\in\Aset }d(s,a')\cdot\mathbbm{1}[ z=\phi(s,a')]} \\
     &= \pirep(z|s),
\end{align}
and we have the desired equality.
\end{proof}
\end{lemma}

\begin{lemma}
\label{lem:decode}
Let $\pirep:\Sset\to\Delta(\Zset)$ be a latent policy in $\Zset$ and $\pidec:\Sset\times\Zset\to\Aset$ be an action decoder, $\pi_{\alpha,\Zset}$ be the marginalization of $\pidec\circ\pirep$ onto $\Zset$:
\begin{equation*}
    \pi_{\alpha,\Zset}(z|s) \defeq \sum_{a\in\Aset, z=\phi(s,a)} (\pidec\circ\pirep)(a|s) = \sum_{a\in\Aset, z=\phi(s,a)}\sum_{\tilde z\in\Zset}\pi_{\alpha}(a|s, \tilde z)\pirep(\tilde z|s).
\end{equation*}
Then for any $s\in\Sset$ we have
\begin{align}
    \dtv(\pirep(s)\|\pi_{\alpha,\Zset}(s)) \le  \max_{z\in\Zset}\dtv(\pi_{\alpha^*}(s, z)\|\pidec(s, z)),
\end{align}
where $\pi_{\alpha^*}$ is the optimal action decoder defined in Lemma~\ref{lem:marginal-opt} (and this holds for any choice of $d$ from Lemma~\ref{lem:marginal-opt}).
\begin{proof}
\begin{align}
    &\dtv(\pirep(s)\|\pi_{\alpha,\Zset}(s)) \\
    =& \frac{1}{2}\sum_{z\in\Zset}\left|\pirep(z|s) - \pi_{\alpha,\Zset}(z|s)\right|\\
    =& \frac{1}{2}\sum_{z\in\Zset}\left|\pirep(z|s) - \sum_{a\in\Aset, z=\phi(s,a)}\sum_{\tilde z\in\Zset}\pi_{\alpha}(a|s, \tilde z)\pirep(\tilde z|s)\right|\\
    =& \frac{1}{2}\sum_{z\in\Zset}\left|\pirep(z|s) - \sum_{a\in\Aset, z=\phi(s,a)}\sum_{\tilde z\in\Zset}\left(\pi_{\alpha}(a|s, \tilde z) - \pi_{\alpha^*}(a|s, \tilde z) + \pi_{\alpha^*}(a|s, \tilde z)\right)\pirep(\tilde z|s)\right|\\
    =& \frac{1}{2}\sum_{z\in\Zset}\left|\sum_{a\in\Aset, z=\phi(s,a)}\sum_{\tilde z\in\Zset}\left(\pi_{\alpha}(a|s, \tilde z) - \pi_{\alpha^*}(a|s, \tilde z)\right)\pirep(\tilde z|s)\right| ~~~~~~~~~~~~~\text{(by Lemma~\ref{lem:marginal-opt})}\\
    \le& \frac{1}{2}\E_{\tilde z\sim\pirep(s)}\left[\sum_{z\in\Zset}\sum_{a\in\Aset, z=\phi(s,a)}\left|\pi_{\alpha}(a|s, \tilde z) - \pi_{\alpha^*}(a|s, \tilde z)\right|\right]\\
    =& \frac{1}{2}\E_{\tilde z\sim\pirep(s)}\left[\sum_{a\in\Aset}\left|\pi_{\alpha}(a|s, \tilde z) - \pi_{\alpha^*}(a|s, \tilde z)\right|\right]\\
    =& \E_{\tilde z\sim\pirep(s)}\left[\dtv(\pidec(s,\tilde z)\|, \pi_{\alpha^*}(s, \tilde z))\right]\\
    \le& \max_{z\in\Zset}\dtv(\pidec(s,z)\|\pi_{\alpha^*}(s, z)),\\
\end{align}
and we have the desired inequality.
\end{proof}
\end{lemma}

\begin{lemma}
\label{lem:tvs}
Let $\pi_{1,\Zset}$ be the marginalization of $\pi_1$ onto $\Zset$ as defined in Lemma~\ref{lem:bottleneck}, and let $\pirep$, $\pidec$, $\pi_{\alpha,\Zset}$ be as defined in Lemma~\ref{lem:decode}, and let $\pi_{\alpha^*,\Zset}$ be as defined in Lemma~\ref{lem:marginal-opt}. For any $s\in\Sset$ we have
\begin{equation}
    \dtv(\pi_{1, \Zset}(s)\|\pi_{\alpha, \Zset}(s)) \le \max_{z\in\Zset}\dtv(\pidec(s,z)\|\pi_{\alpha^*}(s, z)) + \dtv(\pi_{1, \Zset}(s)\|\pirep(s)).
\end{equation}
\begin{proof}
    The desired inequality is achieved by plugging the inequality from Lemma~\ref{lem:decode} into the following triangle inequality:
\begin{equation}
    \dtv(\pi_{1, \Zset}(s)\|\pi_{\alpha, \Zset}(s)) \le \dtv(\pirep(s)\|\pi_{\alpha,\Zset}(s)) + \dtv(\pi_{1, \Zset}(s)\|\pirep(s)).
\end{equation}
\end{proof}
\end{lemma}

Our final lemma will be used to translate on-policy bounds to off-policy.
\begin{lemma}
\label{lem:off-policy}
For two distributions $\rho_1,\rho_2\in\Delta(\Sset)$ with $\rho_1(s)>0 \Rightarrow \rho_2(s) > 0$, we have,
\begin{equation}
    \E_{\rho_1}[h(s)] \le (1 + \dchi(\rho_1\|\rho_2)^\frac{1}{2}) \sqrt{\E_{\rho_2}[h(s)^2]}. 
\end{equation}
\end{lemma}
\begin{proof}
The lemma is a straightforward consequence of Cauchy-Schwartz:
\begin{align}
    \E_{\rho_1}[h(s)] &= \E_{\rho_2}[h(s)] + (\E_{\rho_1}[h(s)] - \E_{\rho_2}[h(s)]) \\
    &= \E_{\rho_2}[h(s)] + \sum_{s\in\Sset}\frac{\rho_1(s) - \rho_2(s)}{\rho_2(s)^{\frac{1}{2}}}\cdot \rho_2(s)^{\frac{1}{2}} h(s) \\
    &\le \E_{\rho_2}[h(s)] + \left(\sum_{s\in\Sset}\frac{(\rho_1(s) - \rho_2(s))^2}{\rho_2(s)}\right)^{\frac{1}{2}}\cdot \left(\sum_{s\in\Sset}\rho_2(s) h(s)^2\right)^{\frac{1}{2}} \\
    &= \E_{\rho_2}[h(s)] + \dchi(\rho_1\|\rho_2)^{\frac{1}{2}}\cdot \sqrt{\E_{\rho_2}[h(s)^2]}.
\end{align}
Finally, to get the desired bound, we simply note that the concavity of the square-root function implies $\E_{\rho_2}[h(s)] \le \E_{\rho_2}[\sqrt{h(s)^2}] \le \sqrt{\E_{\rho_2}[h(s)^2]}$.
\end{proof}

\section{Proofs for Major Theorems}
\subsection{Proof of Theorem~\ref{thm:tabular}}
\begin{proof}
Let $\pi_2 \defeq \pidec\circ\pirep$, we have $\pi_{2,\Zset}(z|s)$ = $\pi_{\alpha,\Zset}(z|s) = \sum_{a\in\Aset,\phi(s,a)=z} (\pidec\circ\pirep)(z|s)$. By plugging the result of Lemma~\ref{lem:tvs} into Lemma~\ref{lem:bottleneck}, we have
\begin{equation}
\err_{d^{\pi_1}}(\pi_1,\pi_2,\Tmodel)] \le \E_{s\sim d^{\pi_1}}\left[\max_{z\in\Zset}\dtv(\pi_{\alpha^*}(s, z)\|\pidec(s,z)) + \dtv(\pi_{1, \Zset}(s)\|\pirep(s))\right].
\end{equation}
By plugging this result into Lemma~\ref{lem:model1}, we have
\begin{align}
\err_{d^{\pi_1}}(\pi_1,\pi_2,\Trans) &\le |\Aset|\E_{(s,a)\sim (d^{\pi_1}, \uniform)}[\dtv(\Trans(s,a)\|\Tmodel(s,a))]\\ 
&+ \E_{s\sim d^{\pi_1}}\left[\max_{z\in\Zset}\dtv(\pi_{\alpha^*}(s, z)\|\pidec(s,z))\right]\\ 
&+ \E_{s\sim d^{\pi_1}}\left[\dtv(\pi_{1, \Zset}(s)\|\pirep(s))\right].
\end{align}
By further plugging this result into Lemma~\ref{lem:performance} and let $\pi_1 = \pitarget$, we have: %\ofir{Can you split this into a few steps? Where does Lemma 9 come in?}
\begin{align}\label{eq:onpolicy}
\Diff(\pidec\circ\pirep,\pitarget) &\leq
\frac{\gamma|A|}{1-\gamma}\cdot\E_{(s,a)\sim (d^{\pi_1}, \uniform)}[\dtv(\Trans(s,a)\|\Trep(s,\phi(s,a))] \nonumber\\
&+ \frac{\gamma}{1-\gamma} \cdot \E_{s\sim\visittarget}[\max_{z\in\Zset}\dtv(\pi_{\alpha^*}(s,z)\|\pidec(s,z))] \nonumber\\
 & + \frac{\gamma}{1-\gamma} \cdot\E_{s\sim \visittarget}[\dtv(\pi_{*,Z}(s)\|\pirep(s))].
\end{align}
Finally, by plugging in the off-policy results of Lemma~\ref{lem:off-policy} to the bound in~\Eqref{eq:onpolicy} and by applying Pinsker's inequality $\dtv(\Trans(s,a)\|\Trep(s,\phi(s,a)))^2\le \frac{1}{2}\dkl(\Trans(s,a)\|\Trep(s,\phi(s,a)))$, we have
\begin{align}
\Diff(\pidec\circ\pirep,\pitarget) &\leq C_1 \cdot\sqrtexplained{\frac{1}{2}\underbrace{\E_{(s,a)\sim\visitrb}\left[\dkl(\Trans(s,a)\|\Trep(s, \phi(s,a)))\right]}_{\displaystyle=\jtrans(\Trep, \phi)}}\nonumber\\
&+ C_2 \cdot\sqrtexplained{\frac{1}{2}\underbrace{\E_{s\sim\visitrb}[\max_{z\in\Zset}
\dkl(\pi_{\alpha^*}(s,z)\|\pidec(s,z))]}_{\displaystyle\approx~\const(\visitrb,\phi) + \jbcdec(\pidec,\phi)}} \nonumber\\
 & + C_3\cdot\sqrtexplained{\frac{1}{2}\underbrace{\E_{s\sim \visittarget}[\dkl(\pi_{*,Z}(s)\|\pirep(s))]}_{\displaystyle =~\const(\pitarget,\phi) + \jbcrep(\pirep)}},
\end{align}

where $C_1 = \gamma|A|(1-\gamma)^{-1}(1+\dchi(\visittarget\|\visitrb)^{\frac{1}{2}})$, $C_2=\gamma(1-\gamma)^{-1}(1+\dchi(\visittarget\|\visitrb)^{\frac{1}{2}})$, and $C_3=\gamma(1-\gamma)^{-1}$. Since the $\max_{z\in\Zset}$ is not tractable in practice, we approximate $\E_{s\sim\visitrb}[\max_{z\in\Zset}
\dkl(\pi_{\alpha^*}(s,z)\|\pidec(s,z))]$ using $\E_{(s,a)\sim\visitrb}[
\dkl(\pi_{\alpha^*}(s,\phi(s,a))\|\pidec(s,\phi(s,a)))]$, which reduces to $\jbcdec(\pidec,\phi)$ with additional constants. We now arrive at
the desired off-policy bound in Theorem~\ref{thm:tabular}.
\end{proof}

\subsection{Proof of Theorem~\ref{thm:sample}}
\begin{lemma}
\label{lem:empirical-tv}
Let $\rho\in\Delta(\{1,\dots,k\})$ be a distribution with finite support. Let $\widehat{\rho}_n$ denote the empirical estimate of $\rho$ from $n$ i.i.d. samples $X\sim\rho$. Then,
\begin{equation}
    \E_n[\dtv(\rho\|\widehat{\rho}_n)] \le \frac{1}{2}\cdot\frac{1}{\sqrt{n}}\sum_{i=1}^k \sqrt{\rho(i)} \le \frac{1}{2}\cdot\sqrt{\frac{k}{n}}.
\end{equation}
\end{lemma}
\begin{proof}
The first inequality is Lemma 8 in~\cite{berend2012convergence} while the second inequality is due to the concavity of the square root function.
\end{proof}

\begin{lemma}
\label{lem:tv-bc}
Let $\Dset\defeq\{(s_i,a_i)\}_{i=1}^n$ be i.i.d. samples from a factored distribution $x(s,a)\defeq\rho(s)\pi(a|s)$ for $\rho\in\Delta(\Sset),\pi:\Sset\to\Delta(\Aset)$. Let $\widehat{\rho}$ be the empirical estimate of $\rho$ in $\Dset$ and $\widehat{\pi}$ be the empirical estimate of $\pi$ in $\Dset$.
Then,
\begin{equation}
    \E_{\Dset}[\E_{s\sim\rho}[\dtv(\pi(s)\|\widehat{\pi}(s))]] \le \sqrt{\frac{|\Sset||\Aset|}{n}}.
\end{equation}
\end{lemma}
\begin{proof}
Let $\widehat{x}$ be the empirical estimate of $x$ in $\Dset$. We have,
\begin{align}
    \E_{s\sim\rho}[\dtv(\pi(s)\|\widehat{\pi}(s))] &= \frac{1}{2}\sum_{s,a} \rho(s)\cdot |\pi(a|s) - \widehat{\pi}(a|s)| \\
    &=\frac{1}{2}\sum_{s,a} \rho(s)\cdot \left|\frac{x(s,a)}{\rho(s)} - \frac{\widehat{x}(s,a)}{\widehat{\rho}(s)}\right| \\
    &\le \frac{1}{2}\sum_{s,a} \rho(s)\cdot \left|\frac{\widehat{x}(s,a)}{\rho(s)} - \frac{\widehat{x}(s,a)}{\widehat{\rho}(s)}\right| + \frac{1}{2}\sum_{s,a} \rho(s)\cdot \left|\frac{\widehat{x}(s,a)}{\rho(s)} - \frac{x(s,a)}{\rho(s)}\right| \\
    &= \frac{1}{2}\sum_{s,a} \rho(s)\cdot \left|\frac{\widehat{x}(s,a)}{\rho(s)} - \frac{\widehat{x}(s,a)}{\widehat{\rho}(s)}\right| + \dtv(x\|\widehat{x}) \\
    &= \frac{1}{2}\sum_{s} \rho(s)\cdot\left|\frac{1}{\rho(s)}-\frac{1}{\widehat{\rho}(s)}\right| \left(\sum_{a} \widehat{x}(s,a)\right) + \dtv(x\|\widehat{x}) \\
    &= \frac{1}{2}\sum_{s} \rho(s)\cdot\left|\frac{1}{\rho(s)}-\frac{1}{\widehat{\rho}(s)}\right| \cdot\widehat{\rho}(s) + \dtv(x\|\widehat{x}) \\
    &=\dtv(\rho\|\widehat{\rho}) + \dtv(x\|\widehat{x}).
\end{align}
Finally, the bound in the lemma is achieved by application of Lemma~\ref{lem:empirical-tv} to each of the TV divergences.
\end{proof}

To prove Theorem~\ref{thm:sample}, we first rewrite Theorem~\ref{thm:tabular} as
\begin{equation}
\Diff(\pirep,\pitarget) \leq
(\ref{eq:tabular-rep})(\phi) + (\ref{eq:tabular-dec})(\phi)
+ C_3\cdot\E_{s\sim \visittarget}[\dtv(\pi_{*,\Zset}(s)\|\pirep(s))],
\end{equation}
where (\ref{eq:tabular-rep}) and (\ref{eq:tabular-dec}) are the first two terms in the bound of Theorem~\ref{thm:tabular}, and $C_3=\frac{\gamma}{1-\gamma}$.

The result in Theorem~\ref{thm:sample} is then derived by setting $\phi = \phi_{\sopt}$ and $\pirep\defeq\pi_{\sopt,\Zset}$ and using the result of Lemma~\ref{lem:tv-bc}.

\subsection{Proof of Theorem~\ref{thm:linear}}
\begin{proof}
The gradient term in Theorem~\ref{thm:linear} with respect to a specific column $\theta_s$ of $\theta$ may be expressed as
\begin{align}
    &\frac{\partial}{\partial\theta_s} \E_{\tilde{s}\sim d^{\pi}, a\sim\pi(\tilde{s})}[(\theta_{\tilde{s}} - \phi(\tilde{s},a))^2] \nonumber \\
    &= -2\E_{a\sim\pi(s)}[\visitpi(s)\phi(s,a)] + 2\visitpi(s)\theta_s \nonumber \\
    &= -2\E_{a\sim\pi(s)}[\visitpi(s)\phi(s,a)] + 2\E_{z=\theta_s}[\visitpi(s)\cdot z],
\end{align}
 and so,
 \begin{align}
    &w(s')^\top \frac{\partial}{\partial\theta_s} \E_{\tilde{s}\sim d^{\pi}, a\sim\pi(\tilde{s})}[(\theta_{\tilde{s}} - \phi(\tilde{s},a))^2] \nonumber\\ 
    &= -2\E_{a\sim\pi(s)}[\visitpi(s)\Tmodel(s'|s,a)] + 2\E_{z=\theta_s}[\visitpi(s)w(s')^\top z].
\end{align}
Summing over $s\in\Sset$, we have:
\begin{align}
 \sum_{s\in\Sset}w(s')^\top \frac{\partial}{\partial\theta_s} \E_{\tilde{s}\sim d^{\pi}, a\sim\pi(\tilde{s})}[(\theta_{\tilde{s}} - \phi(\tilde{s},a))^2]\nonumber\\ = 2\E_{s\sim\visitpi,a\sim\pi(s),z=\theta_s}[-\Tmodel(s'|s,a)+\Trep(s'|s,z)]
\end{align}

Thus, we have:
\begin{align}\label{eq:linear-model}
    \err_{d^{\pi}}(\pi,\pilrep,\Tmodel) &=\frac{1}{2}\sum_{s'\in\Sset} \left| \E_{s\sim\visitpi, a\sim\pi(s),z=\theta_s}[-\Tmodel(s'|s,a) + \Trep(s'|s, z)] \right| \nonumber\\
    &=\frac{1}{4}\sum_{s'\in\Sset}\left|\sum_{s\in\Sset}w(s')^\top \frac{\partial}{\partial\theta_s} \E_{\tilde{s}\sim d^{\pi}, a\sim\pi(\tilde{s})}[(\theta_{\tilde{s}} - \phi(\tilde{s},a))^2] \right| \nonumber\\
    &\le \frac{1}{4}|S|\|w\|_\infty \cdot \left\|\frac{\partial}{\partial\theta} \E_{s\sim d^{\pi}, a\sim\pi(s)}[(\theta_s - \phi(s,a))^2]
    \right\|_1.
\end{align}

Then by combining Lemmas~\ref{lem:performance},~\ref{lem:model1},~\ref{lem:off-policy}, and apply~\Eqref{eq:linear-model} (as opposed to Lemma~\ref{lem:bottleneck} as in the tabular case), we arrive at the desired bound in Theorem~\ref{thm:linear}.
\end{proof}
 
\section{Experiment Details}
\label{app:exp}
\subsection{Architecture}
We parametrize $\phi$ as a two-hidden layer fully connected neural network with $256$ units per layer. A Swish~\citep{ramachandran2017searching} activation function is applied to the output of each hidden layer. We use embedding size $64$ for AntMaze and $256$ for Ant and all DeepMind Control Suite (DMC) tasks after sweeping values of $64$, $256$, and $512$, though we found \method to be relatively robust to the latent dimension size as long as it is not too small (i.e., $\ge64$). The latent skills in temporal skill extraction require a much smaller dimension size, e.g., $8$ or $10$ as reported by \citet{ajay2020opal,pertsch2021guided}. We tried increasing the latent skill size for these work during evaluation, but found the reported value $8$ to work the best. We additionally experimented with different extend of skill extraction, but found the previously reported $t=10$ to also work the best. We implement the trajectory encoder in OPAL, SkiLD, and SPiRL using a bidirectional LSTM with hidden dimension $256$. We use $\beta = 0.1$ for the KL regularization term in the $\beta$ VAE of OPAL (as reported). We also use $0.1$ as the weight for SPiRL and SkiLD's KL divergence terms.

\subsection{Training and Evaluation}
During pretraining, we use the Adam optimizer with learning rate $0.0003$ for $200$k iterations with batch size $256$ for all methods that require pretraining. During downstream behavioral cloning, learned action representations are fixed, but the action decoder is fine-tuned on the expert data as suggested by~\citet{ajay2020opal}. Behavioral cloning for all methods including vanilla BC is trained with learning rate $0.0001$ for $1$M iterations. We experimented with learning rate decay of downstream BC by a factor of $3$ at the $200$k boundary for all methods. We found that when the expert sample size is small, decaying learning rate can prevent overfitting for all methods. The reported results are with learning rate decay on AntMaze and without learning rate decay on other environments for all methods. During the downstream behavioral cloning stage, we evaluate the latent policy combined with the action decoder every $10$k steps by executing $\pidec\circ\pirep$ in the environment for $10$ episodes and compute the average total return. Each method is run with $4$ seeds where each seed corresponds to one set of action representations and downstream imitation learning result on that set of representations. We report the mean and standard error for all methods in the bar and line figures.

\subsection{Modification to SkiLD and SPiRL}
Since SkiLD~\citep{pertsch2021guided} and SPiRL~\citep{pertsch2020accelerating} are originally designed for RL as opposed to imitation learning, we replace the downstream RL algorithms of SkiLD and SPiRL by behavioral cloning with regularization (but keep skill extraction the same as the original methods). Specifically, for SkILD, we apply a KL regularization term between the latent policy and the learned skill prior in the suboptimal offline dataset during pretraining, and another KL regularization term between the latent policy and a learn ``skill posterior'' on the expert data as done in the original paper during downstream behavioral cloning. We do not need to train the binary classifier that SkiLD trains to decide which regularizer to apply because we know which set of actions are expert versus suboptimal in the imitation learning setting. For SPiRL, we apply the KL divergence between latent policy and skill prior extracted from offline data (i.e., using the red term in Algorithm 1 of~\citet{pertsch2020accelerating}) as an additional term to latent behavioral cloning.

\subsection{Dataset Details}
\paragraph{AntMaze.} For the expert data in AntMaze, we use the goal-reaching expert policies trained by~\citet{ajay2020opal} (expert means that the agent is trained to navigate from the one corner of the maze to the opposite corner) to collect $n=10$ trajectories. For the suboptimal data in AntMaze, we use the full D4RL datasets \texttt{antmaze-large-diverse-v0}, \texttt{antmaze-medium-play-v0}, \texttt{antmaze-medium-diverse-v0}, and \texttt{antmaze-medium-play-v0}.

\paragraph{Ant.} For the expert data in Ant, we use a small set of expert trajectories selected by taking either the first $10$k or $25$k transitions from \texttt{ant-expert-v0} in D4RL, corresponding to about $10$ and $25$ expert trajectories, respectively. For the suboptimal data in Ant, we use the full D4RL datasets \texttt{ant-medium-v0}, \texttt{ant-medium-replay-v0}, and \texttt{ant-random-v0}.

\paragraph{RL Unplugged.}
For DeepMind Control Suite~\citep{tassa2018deepmind} set of tasks, we use the RL Unplugged~\citep{gulcehre2020rl} dataset. For the expert data, we take $\frac{1}{10}$ of the trajectories whose episodic reward is among the top $20\%$ of the open source RL Unplugged datasets following the setup in~\citet{zolna2020offline}. For the suboptimal data, we use the bottom $80\%$ of the RL Unplugged dataset. Table~\ref{tab:rlu} records the total number of trajectories available in RL Unplugged for each task ($80\%$ of which are used as suboptimal data), and the number of expert trajectories used in our evaluation.

\begin{table}[h]
    \centering
    \begin{tabular}{l|r|r}\toprule
         Task & \# Total & \# $\Demos$ \\\hline
         cartpole-swingup & $40$ & $2$ \\
         cheetah-run & $300$ & $3$ \\
         fish-swim & $200$ & $1$ \\
         humanoid-run & $3000$ & $53$ \\
         walker-stand & $200$ & $4$ \\
         walker-walk & $200$ & $6$ \\\bottomrule
    \end{tabular}
    \caption{Total number of trajectories from RL Unplugged~\citep{gulcehre2020rl} locomotion tasks used to train CRR~\citep{wang2020critic} and the number of expert trajectories used to train \method. The bottom $80\%$ of \# Total is used to learn action representations by \method.}
    \label{tab:rlu}
\end{table}
\newpage
\section{Additional Empirical Restuls}
\label{app:results}

\begin{figure}[h]
\centering
 \includegraphics[width=\linewidth]{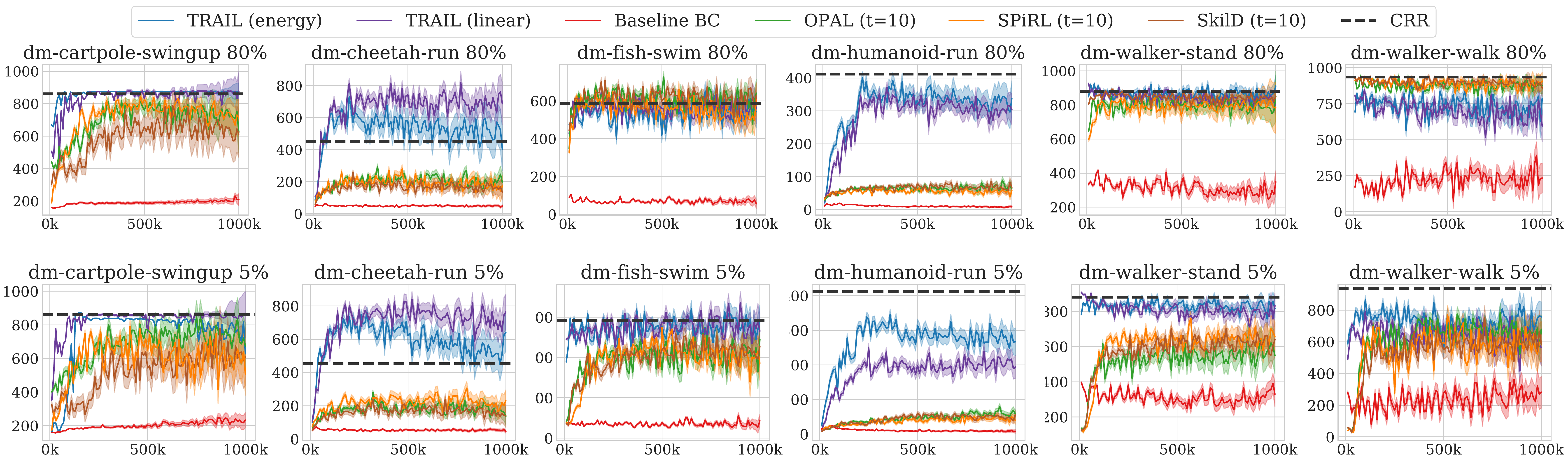}
 \caption{Average task rewards (over $4$ seeds) of \method EBM (Theorem~\ref{thm:tabular}), \method linear (Theorem~\ref{thm:linear}), and OPAL, SkiLD, SPiRL trained on the bottom $80\%$ (top) and bottom $5\%$ (bottom) of the RL Unplugged datasets followed by behavioral cloning in the latent action space. Baseline BC achieves low rewards due to the small expert sample size. Dotted lines denote the performance of CRR~\citep{wang2020critic} trained on the full dataset with reward labels.}
 \label{fig:rlu5}
\end{figure}